\RequirePackage{fix-cm}
\documentclass[smallcondensed]{svjour3}     % onecolumn (ditto)
\smartqed  % flush right qed marks, e.g. at end of proof
 \usepackage{mathptmx}      % use Times fonts if available on your TeX system

\usepackage[round]{natbib}
\usepackage[ruled,vlined]{algorithm2e}
\usepackage[pdftex]{graphicx}
\usepackage{amssymb,amsfonts,amsmath}
\usepackage{enumerate}
\usepackage{proof}
\usepackage{color}
\usepackage{hyperref}

\newtheorem{Theorem}{Theorem}

\newtheorem{Lemma}{Lemma}
\newtheorem{Corollary}{Corollary}
\newtheorem{Proposition}{Proposition}
\newcommand{\field}[1]{\mathbf{#1}}

\newcommand{\z}{\field{z}}

\newcommand{\x}{\field{x}}

\newcommand{\y}{\field{y}}
\newcommand{\w}{\field{w}}
\newcommand{\I}{\field{I}}

\newcommand{\cp}{\xrightarrow{P}}
\newcommand{\cd}{\xrightarrow{d}}
\newcommand{\yh}{\hat{\y}(\x)}
\newcommand{\ys}{\y^{*}(\x)}

 \journalname{Machine Learning}
\begin{document}

\title{Asymptotic consistency and order specification for logistic classifier chains in multi-label learning%\thanks{Grants or other notes
}

\titlerunning{Asymptotic consistency and order specification for logistic classifier chains}        % if too long for running head

\author{Pawe{\l} Teisseyre}


\institute{Pawe{\l} Teisseyre \at
              Institute of Computer Science, Polish Academy of Sciences \\
              Jana Kazimierza 5 01-248 Warsaw, Poland\\
              Tel.: +48-22-380-05-55 \\
              \email{teisseyrep@ipipan.waw.pl}           %  \\
}

\date{Received: date / Accepted: date}

\maketitle

\begin{abstract}
Classifier chains are popular and effective method to tackle a multi-label classification problem.
The aim of this paper is to study the asymptotic properties of the chain model in which the conditional probabilities are of the logistic form. In particular we find conditions on the number of labels and the distribution of feature vector under which the estimated mode of the joint distribution of labels converges to the true mode. 
Best of our knowledge, this important issue has not yet been studied in the context of multi-label learning.
We also investigate how the order of model building in a chain influences the estimation of the joint distribution of labels.  
We establish the link between the problem of incorrect ordering in the chain and  incorrect model specification.
We propose a procedure of determining the optimal ordering of labels in the chain, which is based on using measures of correct specification and allows to find the ordering such that the consecutive logistic models are best possibly specified.
The other important question raised in this paper is how accurately can we estimate the joint posterior probability when the ordering of labels is wrong or the logistic models in the chain are incorrectly specified. The numerical experiments illustrate the theoretical results.
\keywords{classifier chains \and logistic regression \and joint mode estimation \and label ordering \and asymptotic consistency}
\end{abstract}

\section{Introduction}
%\begin{table}
%\begin{tabular}{lll}
%\hline\noalign{\smallskip}
%\end{tabular}
%\end{table}

In multi-label classification the task is to automatically assign an object to multiple categories based on its characteristics. Each object of our interest is described by a feature vector $\x$ belonging to $p$-dimensional space and vector of $K$ labels $\y=(y_1,\ldots,y_K)'$. In this paper we consider binary labels such that $y_k=1$ indicates that the considered object belongs to $k$-th category or has the $k$-th property. 
The issue has recently attracted  significant attention, motivated by an increasing number of applications such as image and video annotation (assigning metadata to digital images or videos), music categorization (assigning various emotions to the songs),  text categorization (assigning various subjects to documents), direct marketing (predicting which products will be purchased), functional genomics (determining the functions of genes and proteins) and medical diagnosis (predicting types of diseases). 

The key problem in multi-label learning is how to utilize label dependencies to improve the classification performance, motivated by which number of multi-label algorithms have been proposed in recent years (see \cite{Madjarov2012} for extensive comparison of several methods).
A naive approach, called binary relevance (BR), which is based on building separate classification models for each label and which does not take into account any dependencies between labels, is deficient in many applications.
 A natural approach, to tackle the issue, is to model the joint conditional probability $P(y_1,\ldots,y_n|\x)$ and then to choose the most probable set of labels for a given $\x$, i.e. the mode of the joint distribution $\y^{*}(\x)=\arg\max_{\y\in\{0,1\}^K}p(\y|\x)$.
This approach minimizes the subset 0/1 loss (\cite{Dembczynskietal2012}), which generalizes the well-known 0/1 loss from the conventional to the multi-label setting.
 Direct modelling of the joint distribution is usually problematic. Classifier chain method (\cite{Readetal2011}, \cite{Dembczynskietal2010}) fixes this problem by using  chain probability rule
\begin{equation}
\label{chain1}
P(\y|\x)=P(y_1,\ldots,y_n|\x)=\prod_{k=1}^{K}P(y_k|y_1,\ldots,y_{k-1},\x),
\end{equation}
which allows to estimate the conditional probabilities in (\ref{chain1}) by using single label classifiers in which $y_k$ is treated as response variable, whereas $y_1,\ldots,y_{k-1},\x$ are treated as explanatory variables.
This gives the approximations of $P(y_k|y_1,\ldots,y_{k-1},\x)$ and thus also the approximation of the joint distribution, which will be denoted by $\hat{P}(\y|\x)$.
 Classifier chains are commonly assigned  to the group of problem-transformation methods (\cite{TsoumakasandKatakis2007}), as in this case the multi-label problem is transformed into several single label problems.
One of the main difficulties of classifier chains is their sensitivity to the ordering of labels during building the models. 
Although the order of conditioning in (\ref{chain1}) is not relevant, the order of training models can affect significantly the estimation accuracy.

In this paper we consider the classifier chains combined with logistic model fitting, i.e. logistic models are used to estimate the conditional probabilities in (\ref{chain1}). Our motivation for studying this particular combination is that the logistic model is one of the most commonly used classifiers in practice. 
The combination of classifier chains and logistic regression has already been used by many authors, in various applications, e.g. in \cite{Kumaretal2013}, \cite{Dembczynskietal2010}, \cite{Dembczynskietal2012}, \cite{Montanes2014}.  
However, there are no theoretical results on large sample properties for such procedures. This paper intends to fill this gap. 
 Moreover, the logistic model has strong theoretical background (\cite{FahrmeirKaufmann1985}, \cite{Fahrmeir1987}), which allows to prove some asymptotic results for classifier chains as well.
 
First, we study the asymptotic properties of classifier chains. We impose the conditions for the distribution of the feature vector $\x$ and number of labels $K$ under which the estimated mode $\hat{\y}(\x)=\arg\max_{\y\in\{0,1\}}\hat{P}(\y|\x)$ converges in probability to the true mode $\y^{*}(\x)$.
The above property can be proved, provided that there exists the "optimal" ordering of labels, such that the consecutive conditional probabilities in (\ref{chain1}) are of the logistic form. However, in practice, the "optimal" ordering is usually unknown or it may not exist at all (this is the case when the data generation scheme is different from logistic one). This leads to further important questions discussed in this paper. What is the influence of the ordering of labels on the estimation of the joint distribution (and thus on the estimation of the mode)? How accurate can we estimate the conditional probabilities in (\ref{chain1}) when the ordering of labels is wrong or the logistic model is incorrectly specified?
It turns out that in this case we attempt to estimate the parameters of assumed logistic model which is closest to the true one in the sense of average Kullback-Leibler distance.
Moreover, we propose a procedure of finding the optimal ordering of labels, which uses measures of correct specification. The procedure allows to determine the ordering such that the consecutive conditional probabilities in (\ref{chain1}) are best possibly specified.        
Although the following results are obtained for logistic model, we believe that some ideas can be expanded to the case when other base-classifiers are used.
The problem of label ordering has been discussed in literature, however the direct influence on the joint mode estimation has not been investigated. \cite{Kumaretal2013} proposed to use log-likelihood function as a measure of correct ordering and beam search to reduce the large number of possible orderings. To eliminate an effect of ordering of labels, \cite{Readetal2009} proposed to average the multi-label predictions over randomly chosen set of permutations. This extension, called ensembled classifier chain (ECC) improves the results of single chain, however it does not answer the question how to find the optimal ordering. 

The rest of the paper is organized as follows. In Section \ref{Classifier chain model} we describe the logistic classifier chain model (LCC), which is motivated by (\ref{chain1}). In Section \ref{Consistency of the joint mode estimation} we state the Theorem concerning the consistency of the joint mode estimation.
The problem of ordering the labels is discussed in Section \ref{Ordering of labels in classifier chain}. 
In Section \ref{Ordering of labels in classifier chain} we also review the measures of correct specification  and the procedure of determining the optimal ordering is presented.
In Section \ref{Inference in classifier chain model}, problem of inference in classifier chain model is described.
 Finally, the results of selected numerical experiments are reported in Section \ref{Empirical evaluation}. 
Section \ref{Conclusions} contains final conclusions and \ref{Proofs} contains proofs.

\section{Logistic classifier chain (LCC) model}
\label{Classifier chain model}
In multi-label learning each object of our interest is described by a pair $(\x,\y)$,
where $\x\in R^{p}$ is the random vector of $p$ explanatory variables (features) and $\y=(y_1,\ldots,y_{K_n})'$ is the vector of $K_n$ binary responses (labels). The first coordinate of $\x$ is equal $1$, which corresponds to the intercept.
We assume that the number of responses depends on the number of observations, which reflects the common situation of large number of labels. Moreover, this assumption can correspond to the situation of sparse labels (equal zero for the majority of cases) - we add the new label at a time when the first non-zero value of this label appears.  
We assume the following data generation scheme.   
First, we assume that there exists a permutation/ordering of labels $\pi^{*}(1),\ldots,\pi^{*}(K_n)$ such that the probability of $y_{\pi^{*}(k)}=1$, given $\x$ and the previous labels $y_{\pi^{*}(1)},\ldots,y_{\pi^{*}(k-1)}$, is always of the logistic form, i.e.
\begin{equation}
\label{marg_distr0}
P(y_{\pi^{*}(k)}=1|\x,y_{\pi^{*}(1)},\ldots,y_{\pi^{*}(k-1)})=\sigma(\z_{k}(\pi^*)'\theta_{k}(\pi^*)),\quad\textrm{for all } k=1,\ldots,K_n,
\end{equation}
where $\sigma(s)=1/(1+\exp(-s))$ is a logistic function; $\z_{k}(\pi^*)=(\x',y_{\pi^{*}(1)},\ldots,y_{\pi^{*}(k-1)})'$ is a combined vector of features and labels with indices $\pi^{*}(1),\ldots,\pi^{*}(k-1)$;  $\theta_{k}(\pi^*)\in R^{p+k-1}$ is a corresponding vector of parameters.
In total we have $pK_n+K_n(K_n-1)/2$ unknown parameters. 
For notation simplicity we will assume that the ordering of labels, for which (\ref{marg_distr0}) holds, is an identity permutation $\pi^{*}(k)=k$ and 
in this case we will write in brief $\z_k$ and $\theta_k$ instead of $\z_{k}(\pi^*)$ and $\theta_{k}(\pi^*)$.
So, in this case, we can simply write
\begin{equation}
\label{marg_distr}
P(y_{k}=1|\x,y_{1},\ldots,y_{k-1})=\sigma(\z_k'\theta_k),\quad\textrm{for all } k=1,\ldots,K_n.
\end{equation}
The joint distribution is of the form
\begin{equation}
\label{joint_distr}
P(\y|\x)=P(y_1,\ldots,y_k|\x)=\prod_{k=1}^{K_n}P(y_k=1|\x,y_1,\ldots,y_{k-1})=\prod_{k=1}^{K_n}\sigma(\z_k'\theta_k)^{y_k}[1-\sigma(\z_k'\theta_k)]^{1-y_k}.
\end{equation} 
The joint distribution corresponding to BR model, with the conditional distributions of the logistic form, can be written in an analogous form to (\ref{joint_distr}), with the only difference that $\z_k=\x$.  
Diagrams on Figure \ref{fig:nets} represent the networks corresponding to LCC model and BR model.
 \begin{figure}[ht!]
\begin{center}$
\begin{array}{ccc}
\includegraphics[scale=0.8]{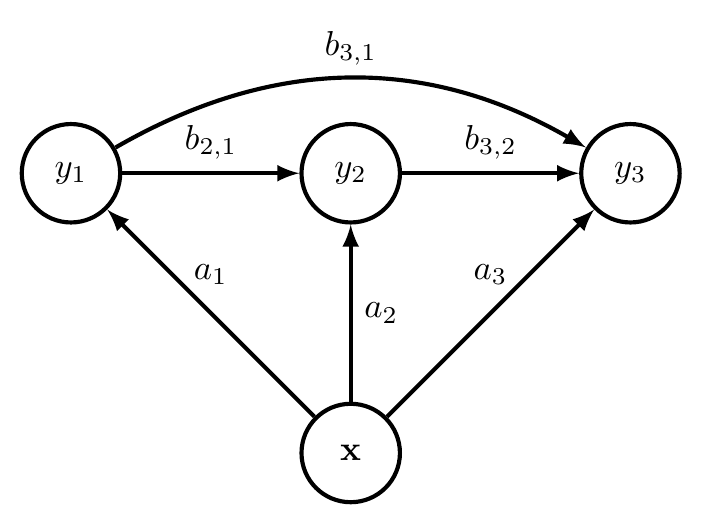} &
\includegraphics[scale=0.8]{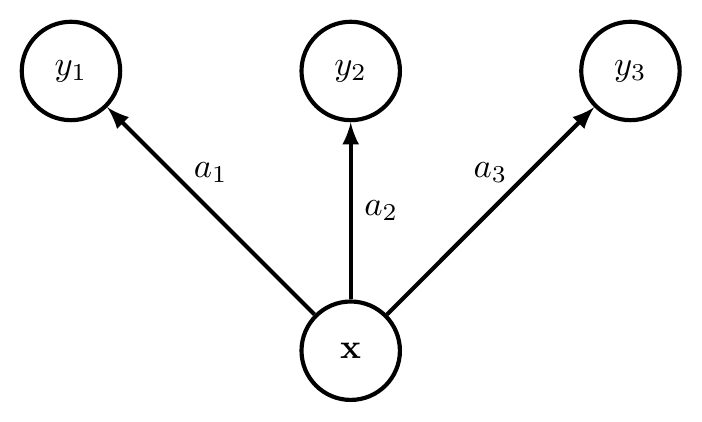} \\
(a) & (b)\\
\end{array}$
\end{center} 
\caption{Diagrams representing data generation schemes in LCC model (a) and BR model (b), in the case of $K_n=3$ labels. The following notation is used: $\theta_k=(a_k',b_k')'$, where $a_k\in R^p$ is a subvector of parameters corresponding to $\x$ and $b_k=(b_{k,1},\ldots,b_{k,k-1})'\in R^{k-1}$ is a subvector of parameters corresponding to $k-1$ first labels in the chain.}
\label{fig:nets}
\end{figure}

As the true joint distribution  is usually unknown, it is estimated based on training examples $\mathcal{D}=(\x^{(i)},\y^{(i)})$, $i=1,\ldots,n$ ($\x^{(i)}\in R^p$, $\y^{(i)}\in\{0,1\}^{K_n}$) which are assumed to be generated from (\ref{marg_distr}).
Observe that conditional distribution (\ref{marg_distr}) can be estimated by fitting logistic regression model (on the set $\mathcal{D}$) in which $y_k$ is a response variable and $\z_k=(\x',\y_{1:(k-1)}')'$ is a vector of explanatory variables. By $\mathcal{M}(y_k,\z_k)$ we will denote a logistic model fitted on set $\mathcal{D}$, in which $y_k$ is a response variable and $\z_k$ is a vector of explanatory variables.

Let $l_{y_k,\z_k}(\cdot)$ be the log-likelihood function calculated for model $\mathcal{M}(y_k,\z_k)$ using $\mathcal{D}$  and
$\hat{\theta}_k=\arg\max_{\theta}l_{y_k,\z_k}(\theta)$ be the maximum likelihood estimator of $\theta_k$.  
By building $K_n$ independent logistic models $\mathcal{M}(y_1,\z_1),\ldots,\mathcal{M}(y_{K_n},\z_{K_n})$, one can estimate parameters $\theta_1,\ldots,\theta_{K_n}$ in the consecutive models in the chain and thus also the joint distribution
\begin{equation}
\label{joint_distr_est}
\hat{P}(\y|\x)=\prod_{k=1}^{K_n}\sigma(\z_k'\hat{\theta}_k)^{y_k}[1-\sigma(\z_k'\hat{\theta}_k)]^{1-y_k}.
\end{equation}

\section{Consistency of the joint mode estimation}
\label{Consistency of the joint mode estimation}
In the following we will assume that the true ordering of labels $(\pi^{*}(1),\ldots,\pi^{*}(K_n))=(1,\ldots,K_n)$, for which (\ref{marg_distr0}) holds, is known and the logistic models $\mathcal{M}(y_1,\z_1),\ldots,\mathcal{M}(y_{K_n},\z_{K_n})$ are fitted according to this order based on training set $\mathcal{D}$.

Denote by $\y^{*}(\x)=\arg\max_{\y\in\{0,1\}^{K_n}}P(\y|\x)$ the true joint mode of the distribution (\ref{joint_distr}), for some new observation $\x\in R^p$ and by
$\hat{\y}(\x)=\arg\max_{\y\in\{0,1\}^{K_n}}\hat{P}(\y|\x)$ the estimated mode for  point $\x$, calculated based on set $\mathcal{D}$. As in multi-label problems we are usually interested in finding the true joint mode, it is worthwhile to approximate it accurately. 
In this section we will give the assumptions under which the true joint mode can be estimated consistently.

Let $s_{y_k,\z_k}(\theta)=\frac{\partial}{\partial\theta}l_{y_k,\z_k}(\theta)$ be the gradient and $H_{y_k,\z_k}(\theta)=-\frac{\partial^2}{\partial\theta\partial\theta'}l_{y_k,\z_k}(\theta)$ be the negative Hessian matrix, based on model $\mathcal{M}(y_k,\z_k)$ and data $\mathcal{D}$.
Let $Z_k$ be the matrix whose $i$-th row is $\z_{k}^{(i)}=(\x^{(i)'},\y^{(i)'}_{1:(k-1)})'$, for $i=1,\ldots,n$.
For logistic model we have 
\begin{equation*}
s_{y_k,\z_k}(\theta)=Z_k'(\w_k-E(\w_k|Z_k)),
\end{equation*}
where $\w_k=(y_k^{(1)},\ldots,y_{k}^{(n)})'$ is a response vector corresponding to $k$-th label.
To prove our main Theorem, we use the following Lemma, which is an important technical tool. The Lemma follows from \cite{Zhang2009}, after noting that binary random variable $w$ satisfies $E\exp(t(w-Ew))\leq \exp(t^2/8)$ and taking $\sigma=1/2$, $\eta=(\epsilon+1/2)^2$ in his Proposition 10.2. 
\begin{Lemma}
\label{Lemma5}
Let $\w$ be $n$-dimensional vector consisting of independent binary variables having not necessarily the same distribution and $Z$ be $n\times q$ fixed matrix. Then for any $\eta>1$ we have
\begin{equation*}
P[||Z'(\w-E\w)||^2\geq tr(Z'Z)\eta]\leq \exp(-\eta/20).
\end{equation*}
\end{Lemma}
The inequality in Lemma \ref{Lemma5} holds also for random $Z$, which can be easily seen by conditioning
\begin{equation*}
P[||Z'(\w-E\w)||^2\geq tr(Z'Z)\eta]=E\{P[||Z'(\w-E(\w|Z))||^2\geq tr(Z'Z)\eta|Z]\}\leq \exp(-\eta/20),
\end{equation*}
where the first expectation in above formula is taken with respect to $Z$. It is seen that Lemma \ref{Lemma5} can be used to bound from above $||s_{y_k,\z_k}(\theta)||$, which is an important part of the proof of Theorem \ref{Theorem1}.
In the following Theorem, we impose conditions under which the true mode $\y^{*}(\x)$ is estimated consistently.
For the proof, see in Section \ref{Proof of Theorem1}.

\begin{Theorem}
\label{Theorem1}
Let $\x$ be a fixed point at which the joint mode is calculated.
Assume that:
\begin{enumerate}
\item 
there exists $\epsilon>0$, such that $P[\y^{*}(\x)|\x]>P[\y|\x]+\epsilon$, for all  
$\y\neq \y^{*}(\x)$,
\item $E||\x^{(i)}||^2<\infty$,
\item $K_n^{4}\log(K_n)/n\to 0$ and $K_n$ is monotonic function of $n$,
\item there exists constant $c_1$ such that
\begin{equation*}
P\left[\min_{1\leq k\leq K_n}\lambda_{\min}[H_{y_k,\z_k}(\theta_k)/n]\geq c_1\right]\to 1,
\end{equation*}
where $\lambda(\cdot)$ is an eigenvalue ($\lambda_{\min}$ is minimal eigenvalue). 
\end{enumerate}
Then, we have that $\yh$ is consistent, i.e. $P[\yh= \ys]\to 1$, as $n\to\infty$.
%Then we have that $P[\hat{\y}(\x)= \y^{*}(\x)]\to 1$, as $n\to\infty$.
\end{Theorem}

The first assumption  ensures that the true mode $\y^{*}(\x)$ is unique for $\x$. The second assumption on the distribution of $\x^{(i)}$ is rather mild and is satisfied for the majority of common distributions. The third one indicates that the number of possible labels should be small in comparison with the number of training examples $n$. 
It can be verified that if for almost every $\y$, $|\{k:y_k\neq 1\}|\leq C<\infty$, with probability tending to $1$, then the third assumption can be weakened to $K_n^{3}\log(K_n)/n\to 0$. This would reflect the situation in which the number of possible labels is large, however only a few of them can be equal to $1$ simultaneously. The last assumption is a regularity condition, analogous to assumption (A4) used in \cite{ChenChen2012} to prove the consistency of Extended Bayesian Information Criterion (EBIC) for logistic regression model. Conditions similar to our last assumption are used in many papers devoted to asymptotic results in logistic regression (see e.g. in \cite{QianField2002} and \cite{SzymanowskiMielniczuk2015}).
Observe that, 
\begin{equation*}
\lambda_{\min}[H_{y_k,\z_k}(\theta_k)/n]\geq \min_{i}[\sigma(\z_{k}^{(i)'}\theta_k)(1-\sigma(\z_{k}^{(i)'}\theta_k))]\lambda_{\min}(Z_k'Z_k/n)
\end{equation*}
and thus condition 4 is implied by the following two assumptions:
\begin{equation}
\label{cond11}
P\left[\min_{1\leq k\leq K_n}\min_{i}[\sigma(\z_{k}^{(i)'}\theta_k)(1-\sigma(\z_{k}^{(i)'}\theta_k))]\geq c_{11}\right]\to 1,
\end{equation}
and
\begin{equation}
\label{cond12}
P\left[\min_{1\leq k\leq K_n}\lambda_{\min}[Z_k'Z_k/n]\geq c_{12}\right]\to  1,
\end{equation}
for some constants $c_{11}$ and $c_{12}$. Above condition (\ref{cond11}) indicates that conditional variance of $y_k$ given $\x,y_1,\ldots,y_{k-1}$ must be bounded from $0$, for all observations and all tasks. The convergence in (\ref{cond12}) is regularity condition on the design matrices. 
Note that (\ref{cond12}) is weaker than assumption stating that  $Z_k'Z_k/n$ converges to some positive definite matrix, which is commonly used assumption in regression analysis (see e.g. \cite{Nishii1984}).

Using the proof of Theorem 1, we can also show the consistency of parameter estimators. Let ${\mathbf \theta}=(\theta_1,\ldots,\theta_{K_n})'\in R^{pK_n+K_n(K_n-1)/2}$ be the vector of parameters corresponding to all models in the chain and let $\hat{{\mathbf\theta}}=(\hat{\theta}_1,\ldots,\hat{\theta}_{K_n})'\in R^{pK_n+K_n(K_n-1)/2}$ be the corresponding vector of estimators. The following Corollary holds (see Section \ref{Proof of Remark1} for the proof).
\begin{Corollary}
\label{Remark1}
Assume that conditions 2,3,4 from Theorem \ref{Theorem1} hold. Then for any $\epsilon>0$
\[
P(||\hat{{\mathbf\theta}}-{\mathbf\theta}||>\epsilon K_n^{-1})\to 0,
\]
as $n\to\infty$.
\end{Corollary}
\section{Ordering of labels in classifier chain}
\label{Ordering of labels in classifier chain}
In practice, the true ordering of labels $(\pi^{*}(1),\ldots,\pi^{*}(K_n))=(1,\ldots,K_n)$, for which (\ref{marg_distr0}) holds, is unknown.
Moreover, it may happen that the ordering, for which the consecutive conditional probabilities are of the logistic form, does not exist at all.
In this section we study the influence of the ordering of labels on the estimation of the joint distribution (\ref{joint_distr}).

To illustrate the problem, consider the following example on real dataset \texttt{emotions} (\cite{Trohidisetal2008}), having $6$ binary labels and $72$ features. 
To obtain convenient visualization, we initially replaced the features by their first principal component, which explains about $88\%$ of variance of the original features.
Figure \ref{fig1} shows estimated joint distributions (for $4$ selected labellings) as functions of the first principal component.
To estimate joint distributions we use classifier chains with logistic regression and two orders of fitting models: $(1,2,3,4,5,6)$ and $(6,5,4,3,2,1)$. 
  Obviously in this case the true ordering of labels is unknown, however it is clearly seen that the order of fitting models in the chain affects the estimate  of the joint probability significantly.  
 
\begin{figure}[ht!]
\centering
    \includegraphics[scale=0.5]{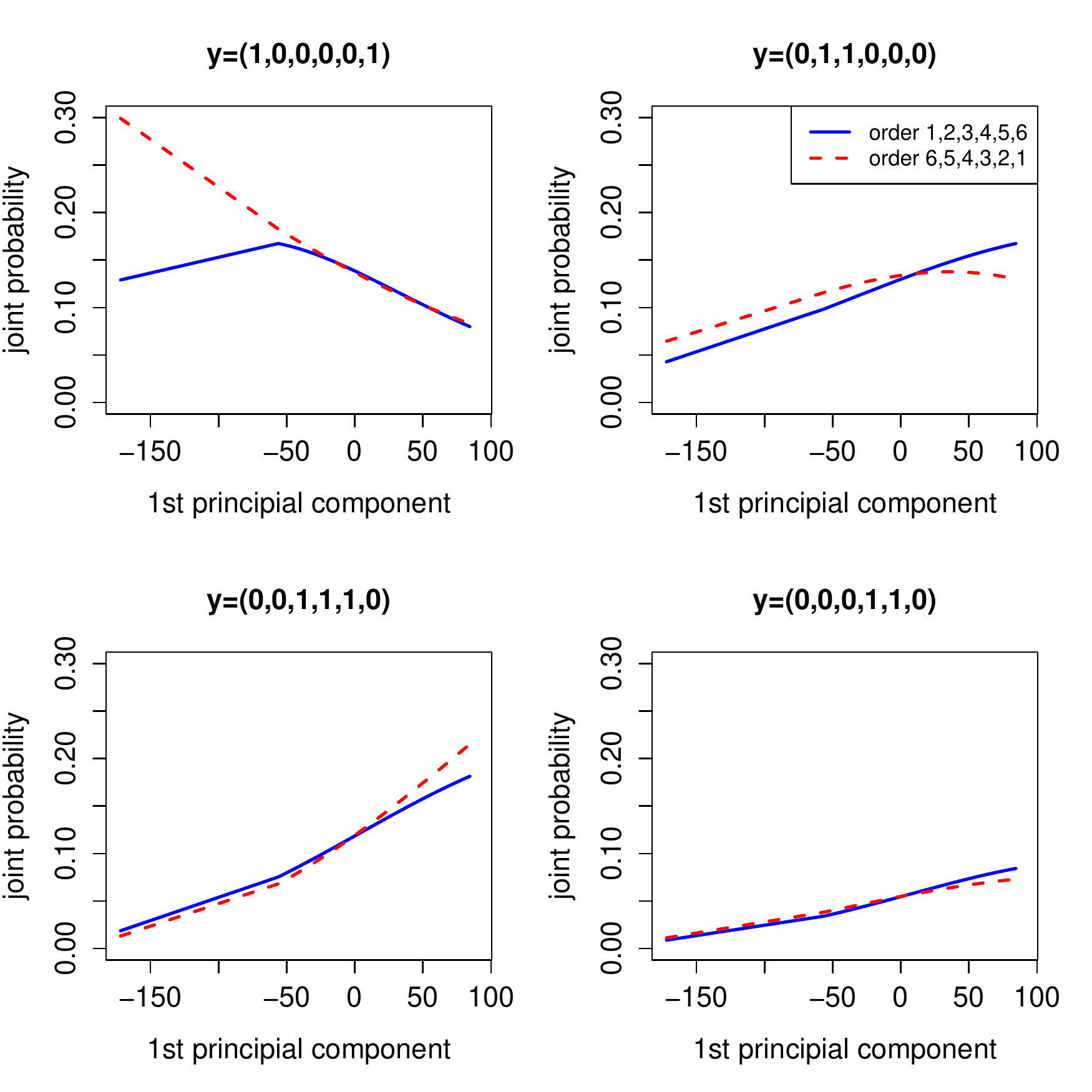}
    \caption{Estimated joint distributions (for $4$ combinations of labels) as functions of the 1st principal component, for \texttt{emotions} dataset.}
    \label{fig1}
\end{figure} 
 
Let $\pi^0$ be assumed ordering of labels and recall that $\z_{k}(\pi^0)=(\x',y_{\pi^{0}(1)},\ldots,y_{\pi^{0}(k-1)})'$.
We consider the situation when logistic models $\mathcal{M}(y_{\pi^{0}(k)},\z_{k}(\pi^0))$, for $k=1,\ldots,K_n$
are fitted according to $\pi^{0}$, whereas the data is generated according to the true ordering $\pi^{*}$.
Unfortunately, in this case we may encounter the problem of incorrect model specification, e.g. we fit the logistic model whereas the true function, from which the data is generated, may not be necessarily a logistic one. 
To assess the problem, consider the simple case of $2$ labels and one feature $x\in R$. Assume that the true ordering is $\pi^{*}=(1,2)$ and the labels are generated as
$P(y_1=1|x)=\sigma(x)$ and $P(y_2=1|x,y_1)=\sigma(x+ay_1)$, where $a\in R$ is a parameter. The joint distribution in this case is of the form
$P(y_1,y_2|x)=\sigma(x)^{y_1}(1-\sigma(x))^{1-y_1}\sigma(x+ay_1)^{y_2}(1-\sigma(x+ay_1))^{1-y_2}$. Assume now that we fit logistic models according to reverse ordering $\pi^{0}=(2,1)$, so we use logistic model to approximate $P(y_2=1|x)$ which is
\begin{equation*}
P(y_2=1|x)=P(y_1=0,y_2=1|x)+P(y_1=1,y_2=1|x)=\sigma(x)+\sigma(x)[\sigma(x+a)-\sigma(x)].
\end{equation*} 
It is easy to verify that there do not exist parameters $v,w\in R$ such that for all $x$, $P(y_2=1|x)=\sigma(v+wx)$, so we cannot approximate 	
accurately the true distribution $P(y_2=1|x)$ using logistic model.

Recall that $\hat{\theta}_{k}(\pi^0)$ is a maximum likelihood estimator from fitted model $\mathcal{M}(y_{\pi^{0}(k)},\z_{k}(\pi^0))$.
We show what is the limit of $\hat{\theta}_{k}(\pi^0)$ in the general situation, when the assumed ordering $\pi^0$ is not necessarily the true one.
\cite{HjortPollard1993} studied the performance of logistic model under incorrect model specification.
See also \cite{CzadoSantner1992} and  \cite{Hjort1988}, where various implications for statistical inference are discussed.
 Using their methodology (see Section 5B in \cite{HjortPollard1993}) it can be shown that for any $\epsilon>0$
\begin{equation}
\label{conv_proj}
P(|\hat{\theta}_{k}(\pi^0)-\tilde{\theta}_{k}(\pi^0)|>\epsilon)\to 0,
\end{equation}
as $n\to\infty$, where
\begin{equation*}
\tilde{\theta}_{k}(\pi^0):=\arg\min_{\gamma\in R^{p+k-1}}E\left\{KL\left[P(y_{\pi^0(k)}=1|\z_{k}(\pi^0)),\sigma(\z_{k}(\pi^0)'\gamma)\right]\right\},
\end{equation*}
and $KL(q_1,q_2):=q_1\log[q_1/q_2]+(1-q_1)\log[(1-q_1)/(1-q_2)]$
%\begin{eqnarray*}
%&&
%KL\left[P(y_{\pi^0(k)}=1|\z_{k}(\pi^0)),\sigma(\z_{k}(\pi^0)'\gamma)\right]=
%\cr
%&&
%P(y_{\pi^0(k)}=1|\z_{k}(\pi^0))\log\frac{P(y_{\pi^0(k)}=1|\z_{k}(\pi^0))}{\sigma(\z_{k}(\pi^0)'\gamma)}+
%[1-P(y_{\pi^0(k)}=1|\z_{k}(\pi^0))]\log\frac{1-P(y_{\pi^0(k)}=1|\z_{k}(\pi^0))}{1-\sigma(\z_{k}(\pi^0)'\gamma)}
%\end{eqnarray*}
 is a Kullback-Leibler divergence. The expectation in (\ref{conv_proj}) is taken with respect to random vector $\z_{k}(\pi^0)$.
So the limit $\tilde{\theta}_{k}(\pi^0)$ minimizes the expected KL divergence from the true distribution $P(y_{\pi^0(k)}=1|\z_{k}(\pi^0))$ to the postulated one with parameter $\sigma(\z_{k}(\pi^0)'\tilde{\theta}_{k}(\pi^0))$.
The convergence in (\ref{conv_proj}) can be interpreted as follows: we find the most accurate approximation of the true probability $P(y_{\pi^0(k)}=1|\z_{k}(\pi^0))$ with respect to the KL divergence, within the space of the logistic models.
When $\pi^0=\pi^*$, the true probability $P(y_{\pi^0(k)}=1|\z_{k}(\pi^0))$ is of the logistic form and then $\tilde{\theta}_{k}(\pi^0)=\theta_{k}(\pi^0)$.
Moreover, we state the proposition which indicates that the expectations of the true distribution and the postulated one, with parameter $\tilde{\theta}_{k}(\pi^0)$,  coincide.
\begin{Proposition}
The following equalities hold
\begin{equation}
\label{moment1}
E[P(y_{\pi^0(k)}=1|\z_{k}(\pi^0))\z_{k}(\pi^0)]=E[\sigma(\z_{k}(\pi^0)'\tilde{\theta}_{k}(\pi^0))\z_{k}(\pi^0)]
\end{equation}
and
\begin{equation}
\label{moment2}
E[P(y_{\pi^0(k)}=1|\z_{k}(\pi^0))]=E[\sigma(\z_{k}(\pi^0)'\tilde{\theta}_{k}(\pi^0))].
\end{equation}
\end{Proposition}
\begin{proof}
Using the fact that $\partial\sigma(\z_{k}(\pi^0)'\gamma)/\partial\gamma=\sigma(\z_{k}(\pi^0)'\gamma)(1-\sigma(\z_{k}(\pi^0)'\gamma))$, it is easy to see that
\begin{eqnarray*}
&&
E\left[\frac{\partial KL\left[P(y_{\pi^0(k)}=1|\z_{k}(\pi^0)),\sigma(\z_{k}(\pi^0)'\gamma)\right]}{\partial\gamma}\right]=
\cr
&&
E[-P(y_{\pi^0(k)}=1|\z_{k}(\pi^0))\z_{k}(\pi^0)+
\sigma(\z_{k}(\pi^0)'\gamma)\z_{k}(\pi^0)],
\end{eqnarray*}
which yields (\ref{moment1}). Equality (\ref{moment2}) follows from (\ref{moment1}) after noting that the first coordinate of $\z_{k}(\pi^0)$, corresponding to intercept, is equal $1$.
\qed
\end{proof}

Hereinafter, we propose a method which aims to find the ordering  such that the consecutive logistic models in the chain are correctly specified. We propose to use a forward procedure such that at each stage we choose the label corresponding to best specified logistic model. Note that we are more interested in evaluation of goodness of specification than evaluation of the goodness of fit.
In order to have such a procedure, it is necessary to define a measure which indicates whether the given model is correctly specified or not.

\subsection{Goodness of specification measures}
\label{Goodness of specification tests}
For convenience, in this section we will focus on the link functions, which relate posterior probabilities to the linear combinations of the features. For the logistic model, the link function is the inverse of $\sigma(\cdot)$.
We discuss the situation of fitting (based on response variable $y$ and some input features $\z$) a logistic model having link function $g_{0}(\mu)=\log(\frac{\mu}{1-\mu})$ when in fact, the correct link function is $g_{*}(\mu)$. The idea is to define more general family of link functions $\mathcal{L}:=\{g(\mu;\alpha):\alpha\in R\}$, depending on some parameter $\alpha$ (or more parameters), in such a way that both assumed link and correct link are the members of $\mathcal{L}$. Having such a family, we may write
$g_{0}(\mu)=g(\mu;\alpha_0)$ and $g_{*}(\mu)=g(\mu;\alpha_*)$, where $\alpha_0$ is known whereas $\alpha_*$ is unknown. 
Using the first-order Taylor expansion about the assumed link we have the approximation
\begin{equation*}
g_{*}(\mu)\approx g_{0}(\mu)+(\alpha_*-\alpha_0)\frac{\partial}{\partial\alpha}g(\mu;\alpha)|_{\alpha=\alpha_{0}}.
\end{equation*} 
Observe that $g_{*}(\mu)=\z'\theta$, where $\z$ is a vector of input variables and $\theta$ is vector of unknown parameters associated with the unknown link and thus the assumed link can be approximated by
\begin{equation}
\label{approx1}
g_{0}(\mu)\approx \z'\theta + w(\alpha_0-\alpha_*),
\end{equation} 
where $w:=\frac{\partial}{\partial\alpha}g(\mu;\alpha)|_{\alpha=\alpha_{0}}$ is so-called carrier variable and $\gamma:=(\alpha_0-\alpha_*)$ is an unknown parameter.
Variable $w$ is unknown but we can approximate it in the following way. We initially fit model with logistic link $g_{0}(\mu)$ and estimate parameters using maximum likelihood method. This yields maximum likelihood estimates $\hat{\theta}$ and thus also $\hat{\mu}$, from which we can approximate $\hat{w}:=\frac{\partial}{\partial\alpha}g(\hat{\mu};\alpha)|_{\alpha=\alpha_{0}}$. 
Now as a measure of correct link specification, we can use deviance statistic, defined as
\begin{equation}
\label{Deviance}
D_{\hat{w}}(y,\z):=2\{l_{y,(\z,\hat{w})}(\hat{\theta}_{e},\hat{\gamma}_{e})-l_{y,\z}(\hat{\theta})\},
\end{equation}
where $\hat{\theta}_{e},\hat{\gamma}_{e}$ denote estimators based on extended model with an additional variable $\hat{w}$.
This requires merely refitting the original logistic model with additional variable $\hat{w}$. 
Observe that $\gamma=0$ if the link is correctly specified. Small value of the deviance indicates that the link is correctly specified, whereas significant departure from $0$ suggest incorrect specification. 
The above reasoning can be easily generalized to the case of more than $1$ carrier, it simply requires refitting the original logistic model with additional $2,3,\ldots$ carriers.
The crucial in the above idea is the choice of the family $\mathcal{L}$. Usually the proposed families are parametrized by one or two parameters. In the literature, various approaches have been explored. Best of our knowledge, the first attempt was made by \cite{Preigbon1980} who use the family defined by
\begin{equation*} 
\mathcal{L}:=\left\{\frac{(\mu/n)^{\alpha-\delta}-1}{\alpha-\delta}-\frac{(1-\mu/n)^{\alpha+\delta}-1}{\alpha+\delta}:\alpha, \delta\in R\right\},
\end{equation*}
where logit link is given by $g_{0}(\mu)=\lim_{\alpha,\delta\to 0}g(\mu;\alpha,\delta)$, which can be seen by applying L'H\^{o}spital's rule twice.
In Preigbon family, $\alpha$ corresponds to the symmetry of the distribution of latent variable associated with a given model from $\mathcal{L}$ and the heaviness of tails of this distributions is parametrized by $\delta$. 
In this case the carriers are $w_1=0.5(\log^{2}(\hat{\mu}/n)-\log^{2}(1\hat{\mu}/n))$ and $w_2=-0.5(\log^{2}(\hat{\mu}/n)+\log^{2}(1\hat{\mu}/n))$.
The comprehensive list of other families can be found in \cite{Stukel1988}. The carriers used in our tests are summarized in Table \ref{tab1}.
\begin{table}
\scriptsize
\centering
\caption{Carriers for computing deviance $D(y,\z)$.}
\label{tab1}
\begin{tabular}{ll}
  \hline 
  Method &  Carriers\\
  \hline
Preigbon& $w_1=0.5(\log^{2}(\hat{\mu}/n)-\log^{2}(1-\hat{\mu}/n))$; $w_2=-0.5(\log^{2}(\hat{\mu}/n)+\log^{2}(1-\hat{\mu}/n))$ \\
Stukel&  $w_1=0.5(\z'\hat{\theta})^2\I(\z'\hat{\theta}\geq 0)$; $w_2=-0.5(\z'\hat{\theta})^2\I(\z'\hat{\theta}< 0)$\\
Prentice&  $w_1=-(1-\hat{\mu})^{-1}\log(\hat{\mu})$; $w_2=-\hat{\mu}^{-1}\log(1-\hat{\mu})$\\
Guerrero-Johnson&  $w_1=0.5(\z'\hat{\theta})^2$\\
Morgan&  $w_1=(\z'\hat{\theta})^3$\\
Aranda (asymmetric) & $w_1=1+\hat{\mu}^{-1}\log(1-\hat{\mu})$\\
\hline
\end{tabular}
\end{table}

\subsection{Procedure for finding the order in classifier chains}
\label{Procedure for finding the order in classifier chains}
The methodology described in the previous section can be used to determine the optimal ordering in LCC.
We propose to use a forward procedure such that at each stage we choose the label corresponding to best specified logistic model.
The single step can be characterized in the following way. We fit logistic models in which features $\x$ and labels found in the previous steps are treated as input variables whereas the candidate labels are treated as response variables. We select the label corresponding to the best specified model. 
%Assume that $\y_{\textrm{sel}}\subseteq \y$ is a set of labels chosen in the previous steps and $\y_{\textrm{act}}=\y\setminus \y_{\textrm{sel}}$, $%\z_{\textrm{act}}=\x\cup\y_{\textrm{sel}}$. From a set $\y_{\textrm{act}}$, we select label $y_{\hat{k}}$ such that
%\begin{equation*}
%\hat{k}=\arg\min_{k\in\y_{\textrm{act}}}D(y_k,\z_{\textrm{act}}),
%\end{equation*}
%where $D$ is defined in (\ref{Deviance}). 
The overall scheme is shown in Algorithm \ref{alg1}.

\begin{algorithm}[ht!]
\caption{Pseudo-code of the procedure for finding the order in classifier chains.}

\SetKwInOut{Input}{Input}
\SetKwInOut{Initialize}{Initialize}
\SetKwInOut{Output}{Output}

\KwData{training set $\mathcal{D}=(\x^{(i)},\y^{(i)})$, $i=1,\ldots,n$}
\Initialize{

$\pi:=\emptyset$

$I:=(1,\ldots,K_n)$

$\z_{\textrm{act}}:=\x$
}
  \For{k $\leftarrow$ $1$ \KwTo $K_n$}{
  
	$\hat{k}=\arg\min_{k\in I}D_{\hat{w}}(y_k,\z_{\textrm{act}})$, where $D$ is defined in (\ref{Deviance})
	   
	$\z_{\textrm{act}}\leftarrow\z_{\textrm{act}}\cup \{y_{\hat{k}}\}$ 
	
	$I\leftarrow I \setminus \{\hat{k}\}$  
   
	$\pi\leftarrow\pi\cup \{\hat{k}\}$

 }
\KwOut{ordering of labels $\pi$}
\label{alg1}
\end{algorithm}

In the following, we will give the heuristic justification of the above procedure.
Consider the general situation in which we would like to test whether the given variable $w$ is significant in the fitted logistic model. 
 Assume that $P(y=1|\z,w)=\sigma(\z'\theta+w\gamma)$, where $(\z',w)'$ is a vector of explanatory variables and $\theta,\gamma$ are unknown parameters. 
Let
\begin{equation}
\label{Deviance_true}
D_{w}(y,\z):=2\{l_{y,(\z,w)}(\hat{\theta}_{e},\hat{\gamma}_{e})-l_{y,\z}(\hat{\theta})\}
\end{equation}
be a deviance statistic,
where
$\hat{\theta}$ is an estimator based on smaller model (with variable $w$ omitted), whereas
$\hat{\theta}_{e},\hat{\gamma}_{e}$ denote estimators based on extended model with an additional variable $w$. Let $D_{\gamma=0}$ denotes deviance (\ref{Deviance_true}) corresponding to the case $\gamma=0$ and $D_{\gamma\neq 0}$ be the analogous quantity, corresponding to the case $\gamma\neq 0$. 
\begin{Theorem}
\label{Theorem2}
Assume that:
\begin{enumerate}
\item $E||\z||^2<\infty$,
\item there exists constant $c_1$ such that
\begin{equation*}
P\left[\lambda_{\min}[H_{y,(\z',w)'}(\theta,\gamma)/n]\geq c_1\right]\to 1
\end{equation*}
($\lambda_{\min}$ is a minimal eigenvalue). 
\end{enumerate}
Then, $P[D_{\gamma\neq 0}>D_{\gamma=0}]\to 1$, as $n\to\infty$.
\end{Theorem}
For the proof, see in Section \ref{Proof of Theorem2}. Part of the proof concerning the asymptotic performance of $D_{\gamma=0}$ follows from \cite{Fahrmeir1987}, whereas the asymptotic convergence of $D_{\gamma\neq 0}$, under the above conditions, was proved in \cite{Teisseyre2013}. 
Observe that assumptions 1, 2 in above Theorem are analogous to assumptions 2, 4 in Theorem \ref{Theorem1}. The above Theorem is formulated  for the case when there is only one additional variable $w$, but it can be obviously expanded to the case of more additional variables.

The above Theorem aims to justify a procedure described in Algorithm \ref{alg1}. 
Namely, deviance $D_{\gamma=0}$ corresponds to the correct specification, whereas $D_{\gamma\neq 0}$ corresponds to the incorrect specification.
Note however, that there are  small differences between situation described in Theorem \ref{Theorem2} and setting described in Section \ref{Goodness of specification tests} (and applied in Algorithm \ref{alg1}). First, observe that the extended model in (\ref{Deviance}) is not exactly a true one (as in (\ref{Deviance_true})), because  it is based on approximation given in (\ref{approx1}). Secondly, observe that an additional variable $\hat{w}$ in (\ref{Deviance}) depends on the estimator $\hat{\theta}$ calculated for the smaller model, which is not the case in (\ref{Deviance_true}).

\section{Inference in classifier chain model}
\label{Inference in classifier chain model}
The classifier chain model allows to estimate the joint probability distribution. Having estimated the joint distributions, one would usually like to infer from this distribution, i.e. to make a prediction for some new instance $\x$. The Bayes optimal rule is the mode of the joint distribution $\y^{*}(\x)=\arg\max_{y\in\{0,1\}^{K_n}}P(\y|\x)$, which  is estimated by $\hat{\y}(\x)=\arg\max_{\y\in\{0,1\}^{K_n}}\hat{P}(\y|\x)$, based on training set. In Theorem \ref{Theorem1} we have shown that for classifier chain model,  $\hat{\y}(\x)$ converges to  $\y^{*}(\x)$ in probability, even in the case of large number of labels, provided that the "true" ordering of labels is known and certain assumptions on the design matrix and number of labels are satisfied.
However, from the practical point of view, the problematic is computation of  $\hat{\y}(\x)$ as it requires evaluation of $2^{K_n}$ possible labellings for one new instance $\x$. This approach is often referred to as "exhaustive inference".
The "exhaustive inference" is limited to data sets with a small to moderate number of labels, say, not more than about 15.
In "multi-label community", the method which combines the classifier chains with "exhaustive inference" is called Probabilistic classifier chains (PCC), see in \cite{Dembczynskietal2010} and \cite{Readetal2011}. 
Chain structure of the model, suggests to use "greedy inference" by successively choosing the
most probable label according to each of the classifiers’ predictions, i.e. when the assumed order of labels is $\pi^0$, then we apply the following procedure
\begin{itemize}
\item find  $\hat{y}_{\pi^{0}(1)}=\arg\max_{y\in\{0,1\}}\hat{P}(y_{\pi^{0}(1)}=y|\x)$,
\item find  $\hat{y}_{\pi^{0}(2)}=\arg\max_{y\in\{0,1\}}\hat{P}(y_{\pi^{0}(2)}=y|\hat{y}_{\pi^{0}(1)},\x)$,
\item $\ldots$
\item find  $\hat{y}_{\pi^{0}(K_n)}=\arg\max_{y\in\{0,1\}}\hat{P}(y_{\pi^{0}(K_n)}=y|\hat{y}_{\pi^{0}(1)},\ldots,\hat{y}_{\pi^{0}(K_n-1)},\x)$,
\end{itemize}    
which requires $K_n$ operations instead of $2^{K_n}$. 
This approach, introduced originally in \cite{Readetal2009}, is usually simply referred to as classifier chains (CC). Since the naming can be a little confusing, we stress that the learning stage is the same in both PCC and CC, whereas the difference is in the inference stage. 
 Note that the "greedy inference" does not guarantee that the mode will be identified correctly, even when the true joint distribution is known. Consider an example with two binary labels $y_1$ and $y_2$ and the ordering $\pi^0=(1,2)$. Let $P(y_1=1|\x)=0.6$, $P(y_1=0|\x)=0.4$, $P(y_2=1|\x,y_1=1)=P(y_2=0|\x,y_1=1)=0.5$ and $P(y_2=1|\x,y_1=0)=0.9$. Then the "greedy inference" yields labelling $\y=(1,1)'$ (or $\y=(1,0)'$), whereas the true mode is $\y^*=(0,1)'$. 
The other problem with "greedy inference", addressed in \cite{Sengeetal2012}, is an error propagation along the chain. Since each classifier relies on its predecessors, a false prediction might be propagated. In a consequence, the labels which are placed at the end of the chain, carry the highest errors.    
  
The problem of inference from the joint distribution can be visualise using the binary tree, in which nodes represent labels, edges represent conditional probabilities and leaves represent particular labellings.
Figure \ref{fig_tree} shows  an example tree for two labels.
The inference task is equivalent to finding the optimal path in a rooted, complete binary tree of height $K_n$. 
The "exhaustive inference" requires evaluation of all $2^{K_n}$ paths in the tree, on the other hand, the "greedy inference" corresponds to one chosen path.    

\begin{figure}[ht!]
\centering
    \includegraphics[scale=0.8]{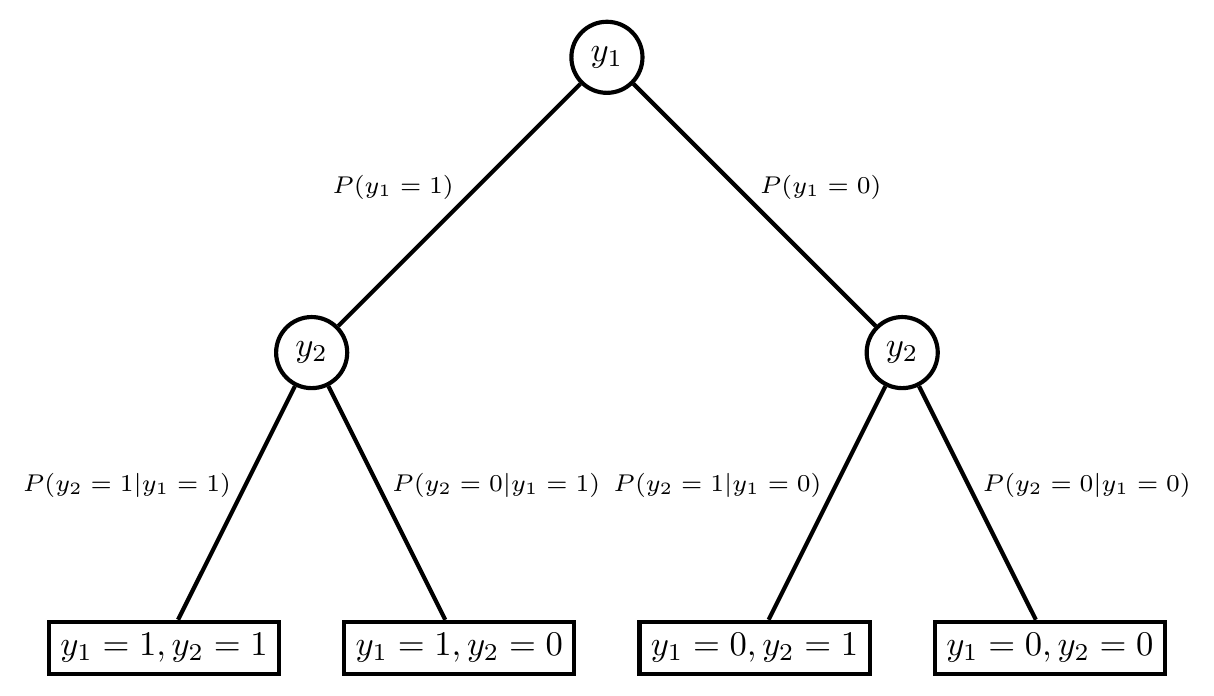}
    \caption{Binary tree used for inference with two labels.}
 \label{fig_tree} 
\end{figure} 
To overcome the problem of time complexity of "exhaustive inference", from one side and the possible lower accuracy of "greedy inference", from the other side, some authors develop approximate inference schemes that trade off accuracy against efficiency in a reasonable way. For example, \cite{Kumaretal2013} proposed using beam search algorithm for inference to explore the binary tree.   
The idea is to keep $b$ candidate solutions at each level of the tree, where $b$ is a user-defined parameter known as
the beam width, which represent the best partial solutions seen thus far. Then the tree is explored in a breadth-first fashion using these solutions. "Beam search inference" requires $bK$ operations and allows to explore more paths in the binary tree, than for "greed inference". Of course, as in the case of "greedy inference", this approach may also not result in finding the optimal label vector. Observe that "Beam search inference" with $b=1$ is equivalent to "greedy inference", whereas $b=\infty$ is equivalent to "exhaustive inference". 
The other simple solution is to limit the inference only to the label combinations that appear in the training set. 
Since the  inference problem is not a main issue of this paper, we limit our empirical experiments to two basic methods: "exhaustive inference" and "greedy inference".  
\section{Empirical evaluation}
\label{Empirical evaluation}
\subsection{Performance of goodness of specification measures}
\label{Performance of goodness of specification measures}
The aim of the first experiment is to compare the goodness of specification tests described in Section \ref{Goodness of specification tests}. We use 10  models generated in the following way. We first generate feature vector $\x$ from uniform distribution on $[-4,4]$ and then the consecutive labels $y_1,\ldots,y_{K_n}$ are generated using classifier chain model (\ref{marg_distr}) according to the order $\pi^{*}(k)=k$. The parameters corresponding to the considered models are summarized in Table \ref{tab2}. The first two simple models match the one discussed at the beginning of Section \ref{Ordering of labels in classifier chain}, with $a=3$ and $a=5$, respectively. The structure of the models was chosen so as to investigate what is the influence of the: number of labels, number of features, values of parameters corresponding to features and values of parameters corresponding to labels. 
We use the procedure described in \ref{Procedure for finding the order in classifier chains} in combination with methods summarized in Table \ref{tab1}.
For each $n=50,100,\ldots$, we repeat $200$ times: generate data with $n$ instances and apply the procedure for finding the optimal ordering.
This enables to estimate the probability of choosing the correct ordering. Figure \ref{fig1} shows the results. The dashed line corresponds to the random choice of the optimal ordering. We also tested using minus log-likelihood function (this measure was used in \cite{Kumaretal2013}) instead of $D(y_k,\z_{\textrm{act}})$ in Algorithm \ref{alg1}. This gives very poor results (comparable with the random choice), so we do not present the results in the figures. 
 The structure of the model affects the probability of choosing the correct ordering.
Observe that, in all cases, the results are significantly better than the baseline (random choice).
The method proposed by Morgan performs poorly in the majority of cases.
The methods, which use two carriers (Preigbon, Stukel, Prentice), usually outperform those which use only one carrier. This is clearly seen for models (M4) and (M6). This phenomenon can be explained by the fact that the methods from the first group are able to represent wider family of possibly misspecified models.
In the case of model (M12), the probabilities are relatively low (about $0.01$ for $n=4000$ and Preigbon method), but they are still much larger than the random choice, which is $1/K_n!\approx 10^{-7}$. On the other hand, in this case, the use of Preigbon methods improves the probability of correct mode selection significantly (see in the next section).  

Note also that the probability of choosing the correct ordering decreases when the number of labels increases (compare models (M3) and (M4) or (M5) and (M10))
and
the probability of choosing the correct ordering decreases when the absolute values of parameters corresponding to labels decrease (compare models (M8) and (M9)).
\begin{figure}[ht!]
\begin{center}$
\begin{array}{ccc}
\includegraphics[scale=0.28]{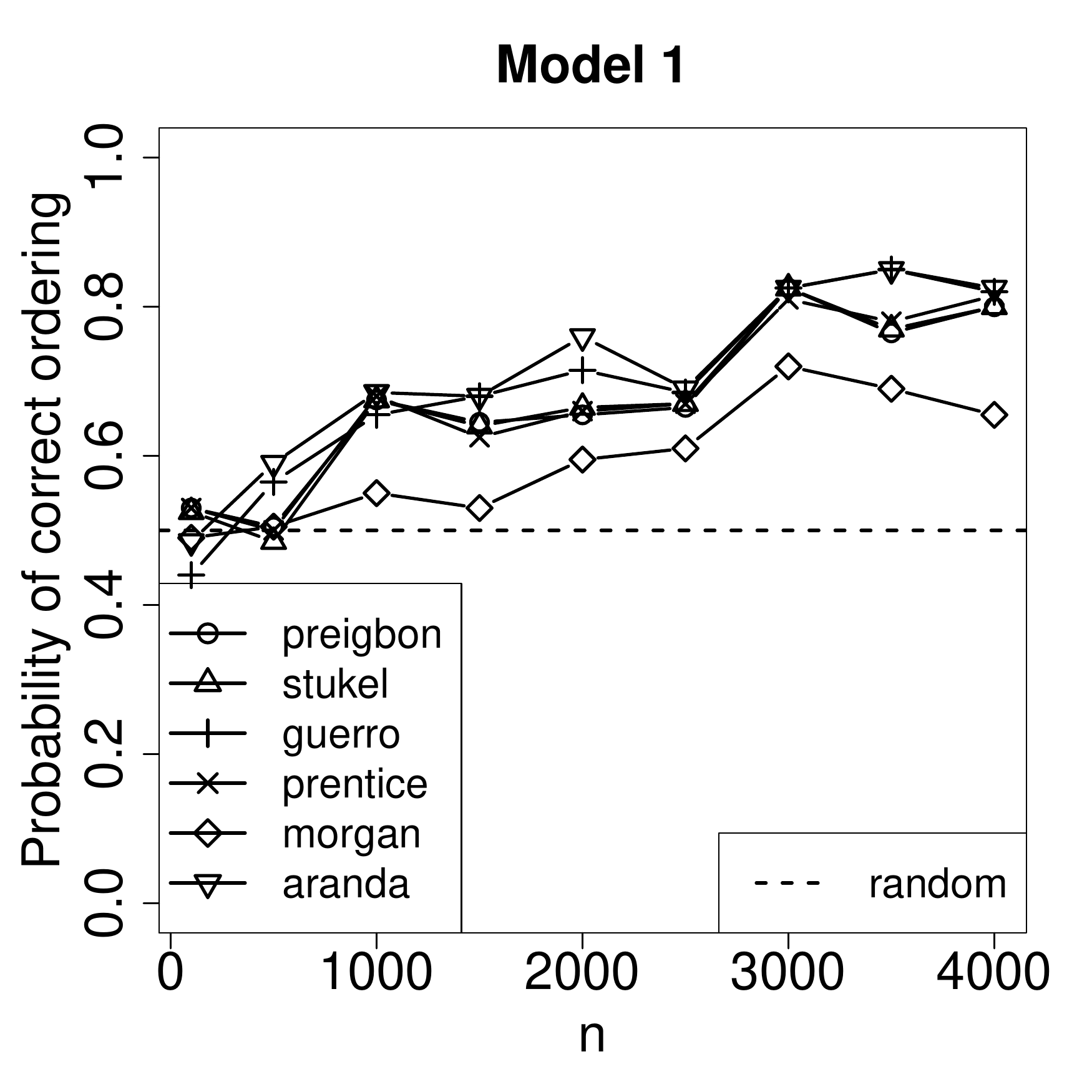} &
\includegraphics[scale=0.28]{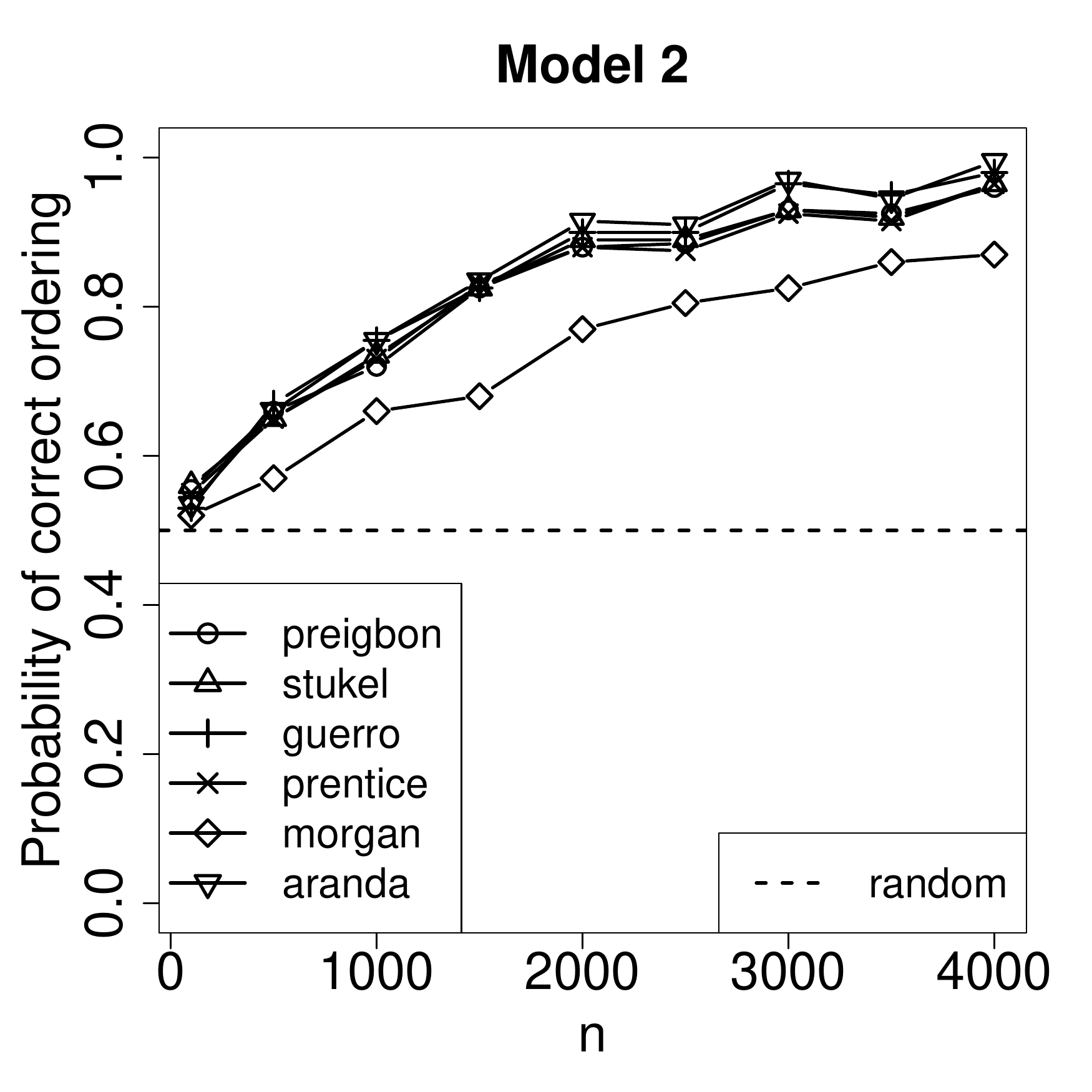} &
\includegraphics[scale=0.28]{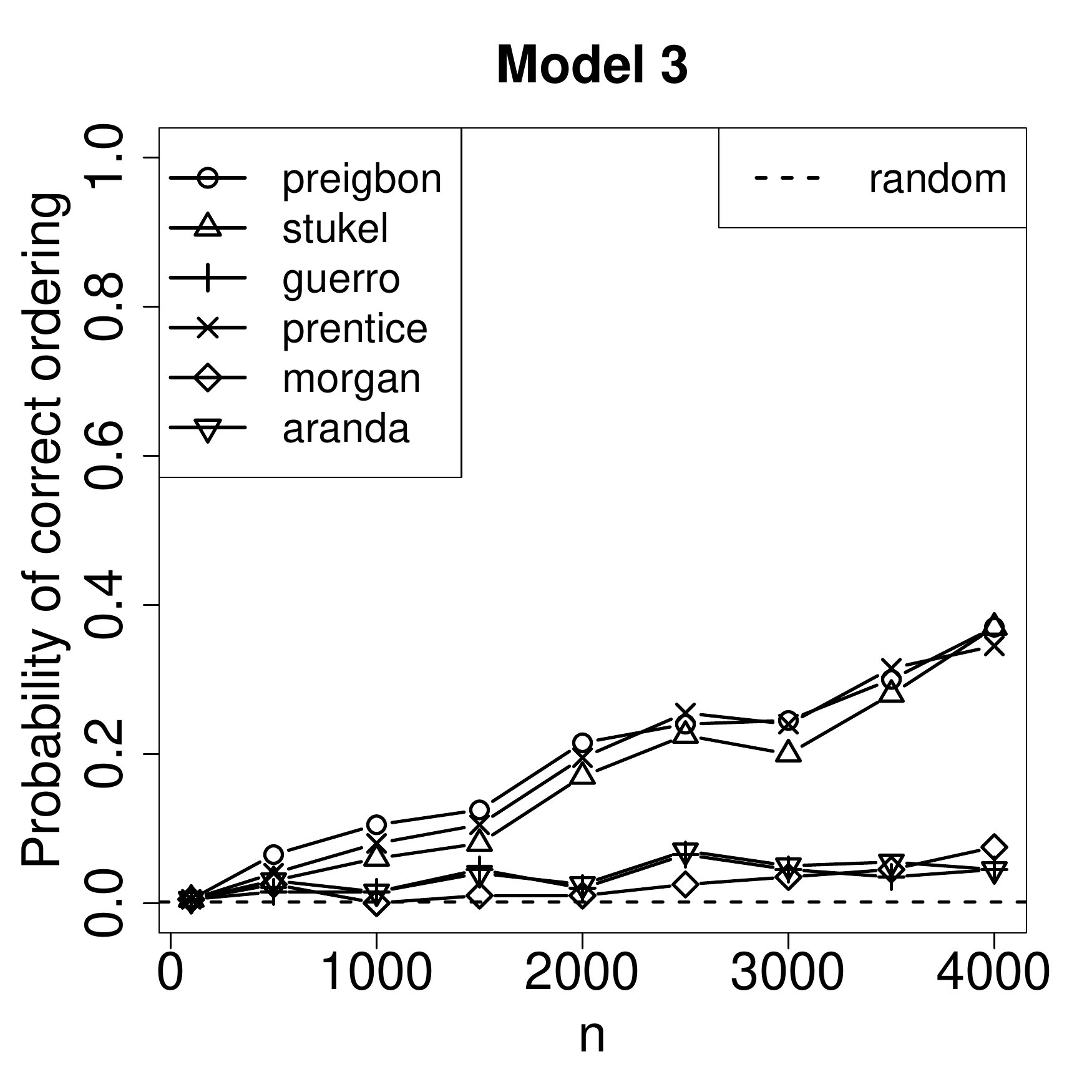} \\
\includegraphics[scale=0.28]{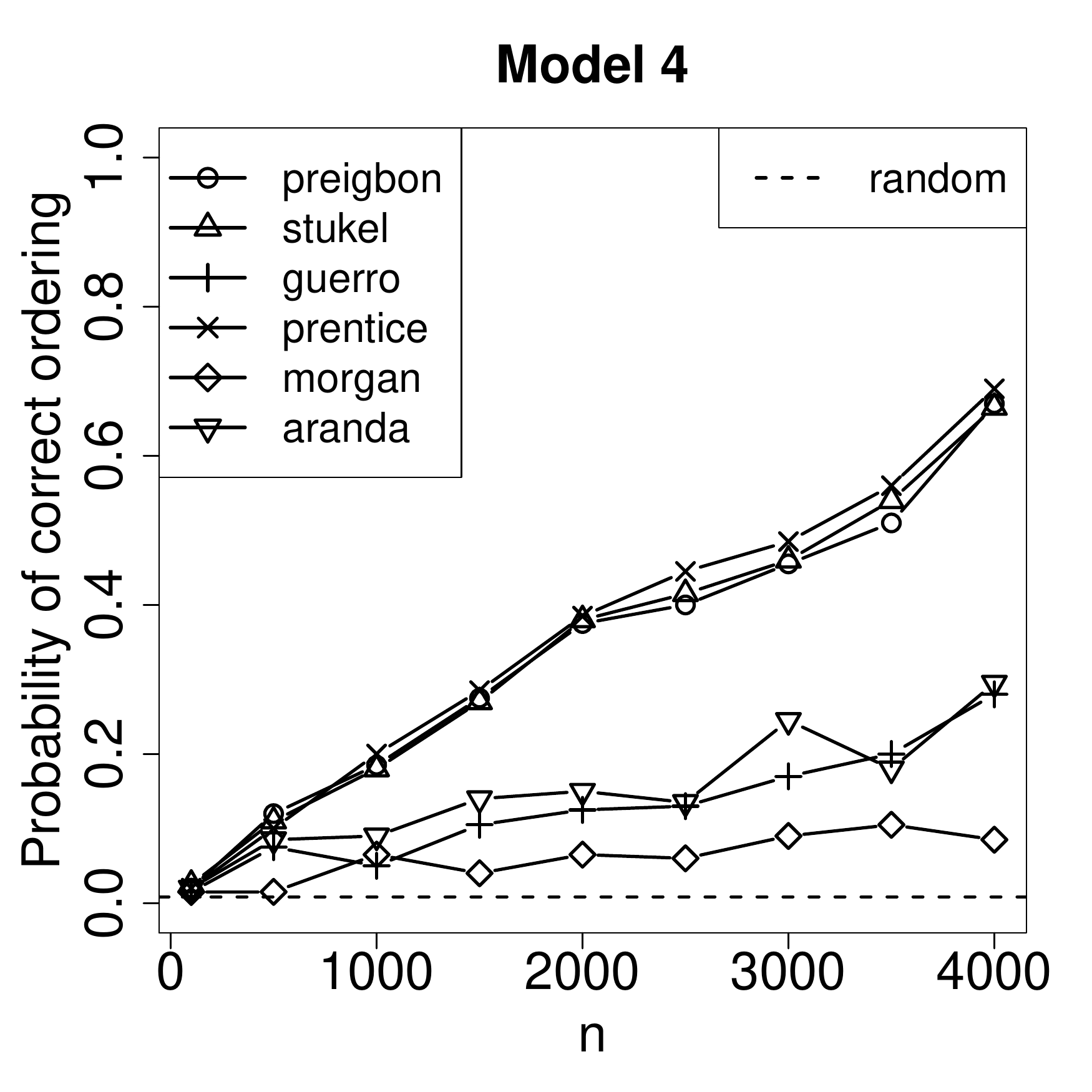} &
\includegraphics[scale=0.28]{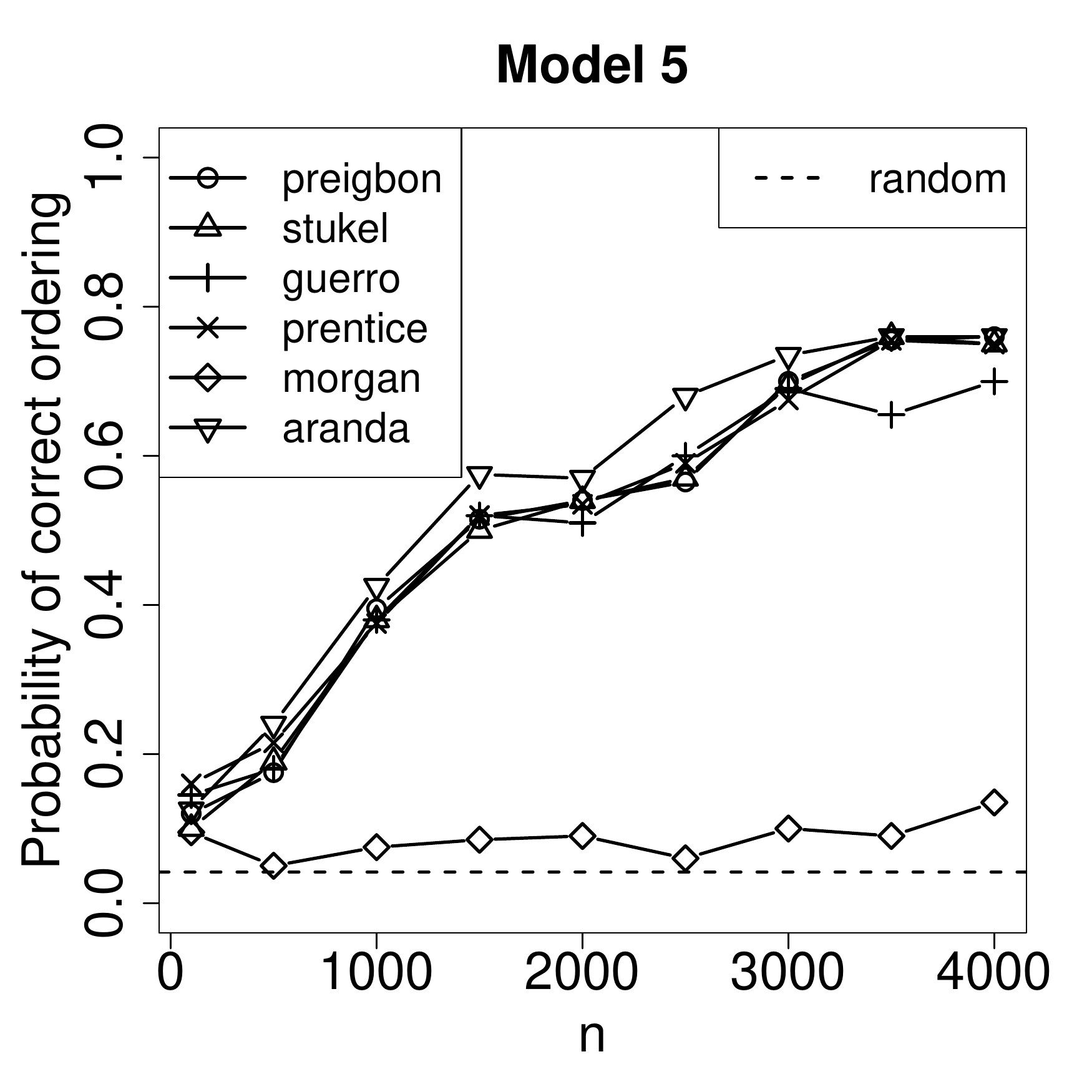} &
\includegraphics[scale=0.28]{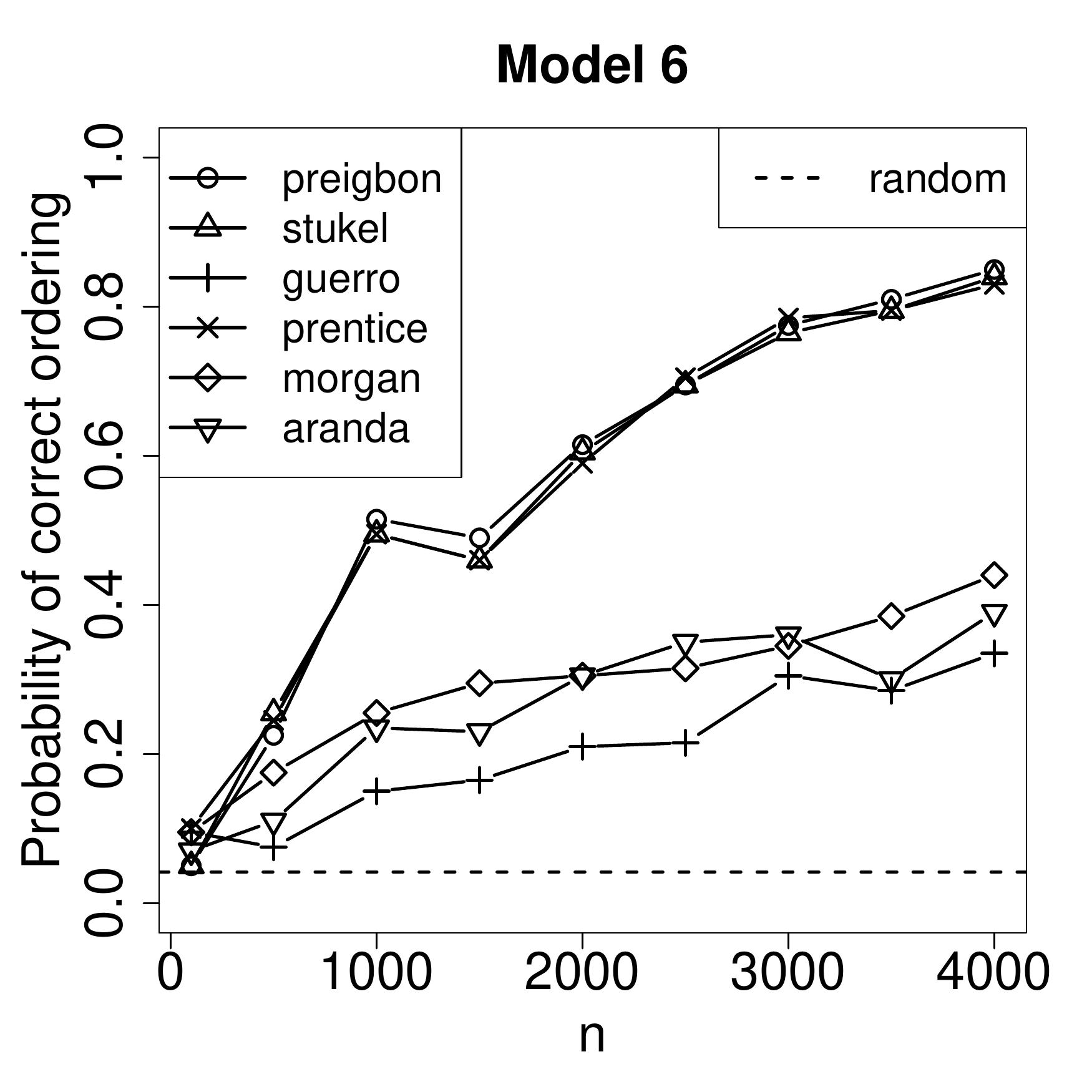} \\
\includegraphics[scale=0.28]{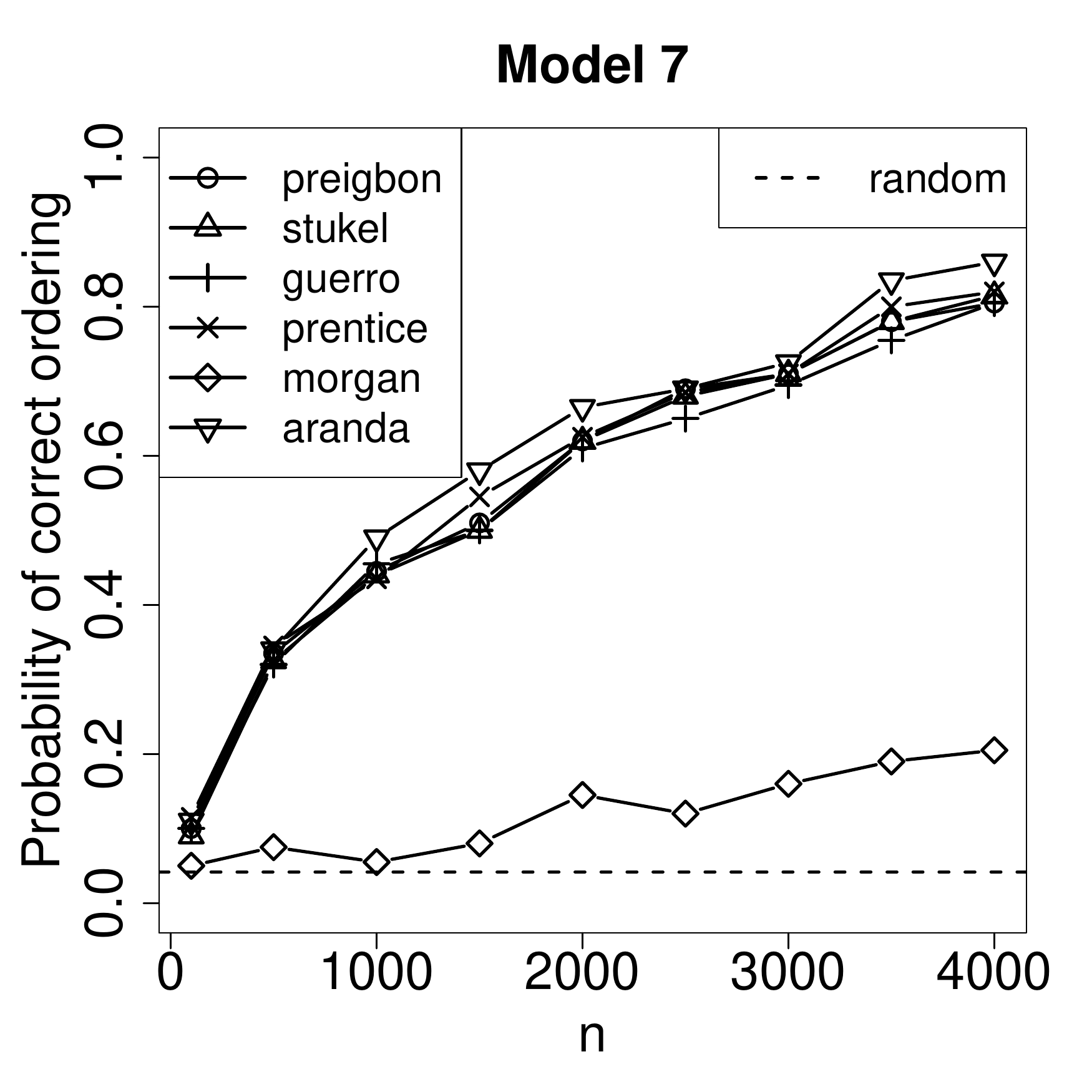} &
\includegraphics[scale=0.28]{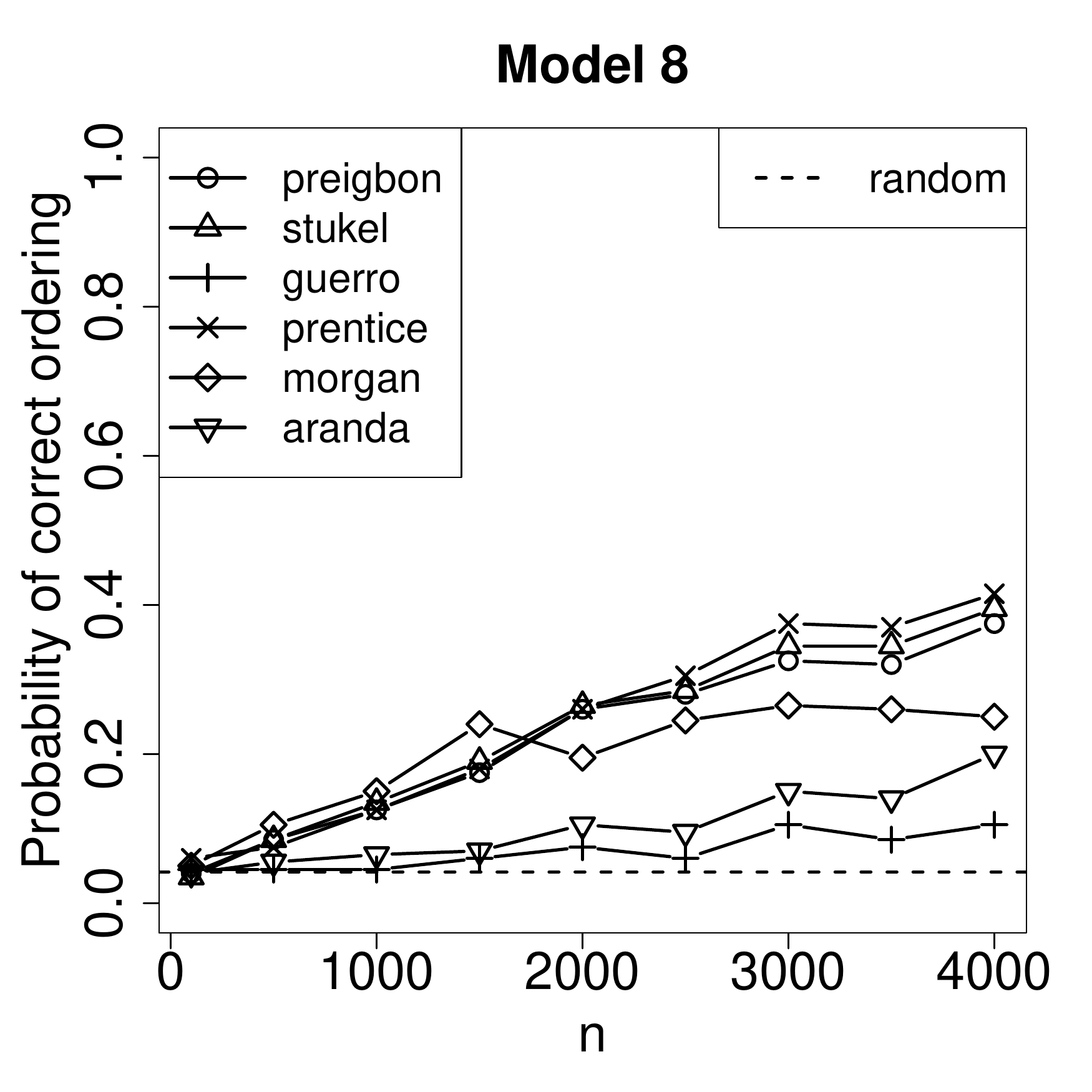} &
\includegraphics[scale=0.28]{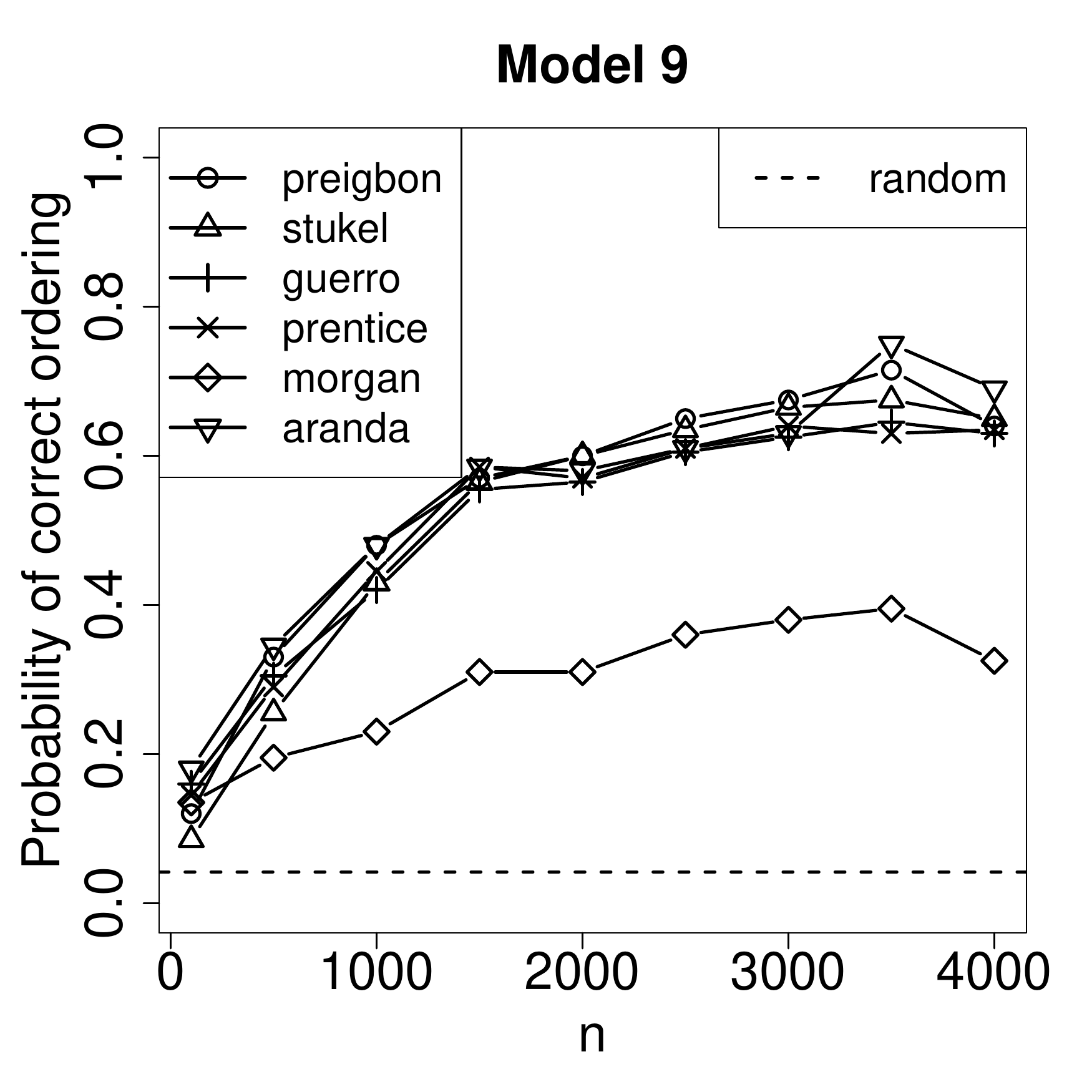} \\
\includegraphics[scale=0.28]{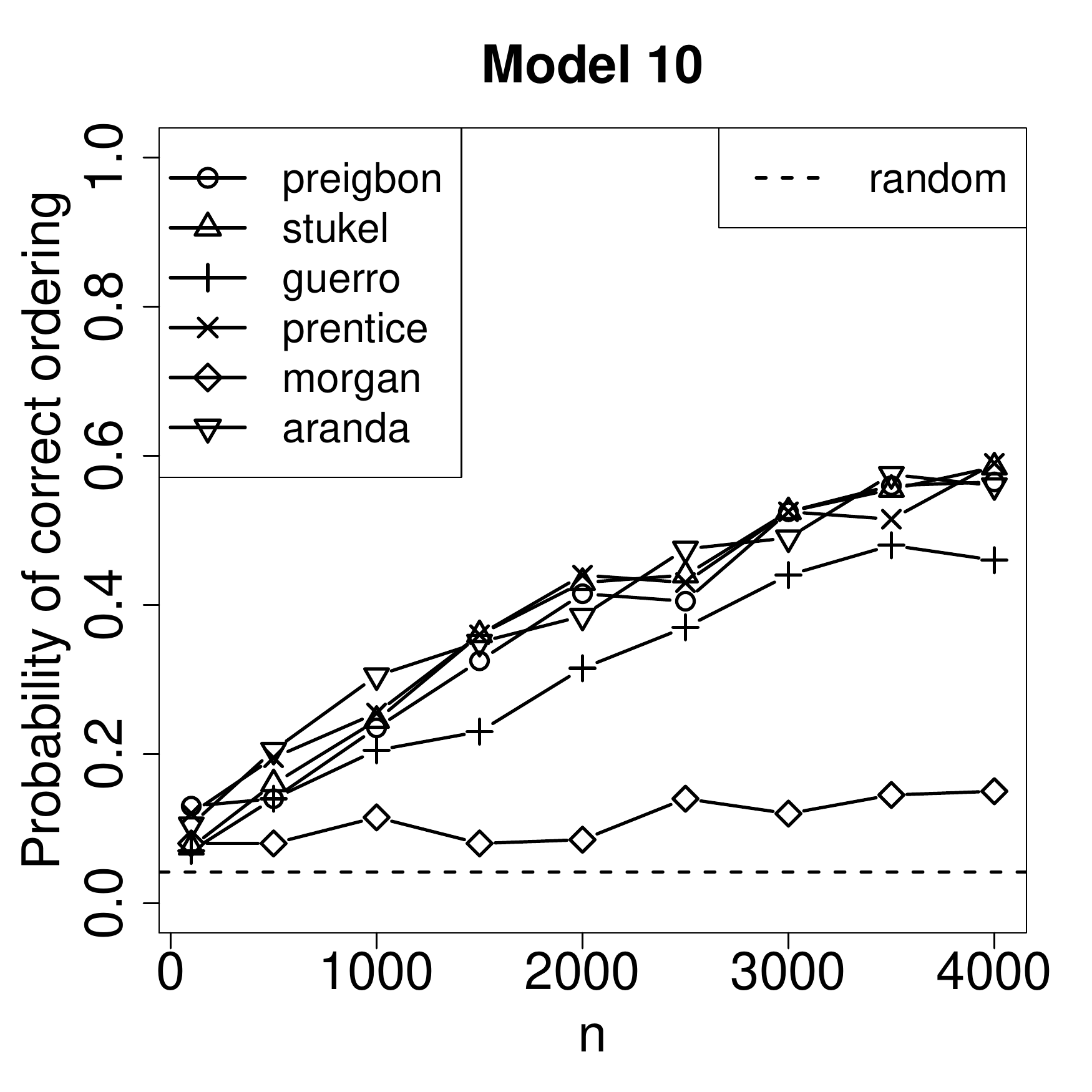} &
\includegraphics[scale=0.28]{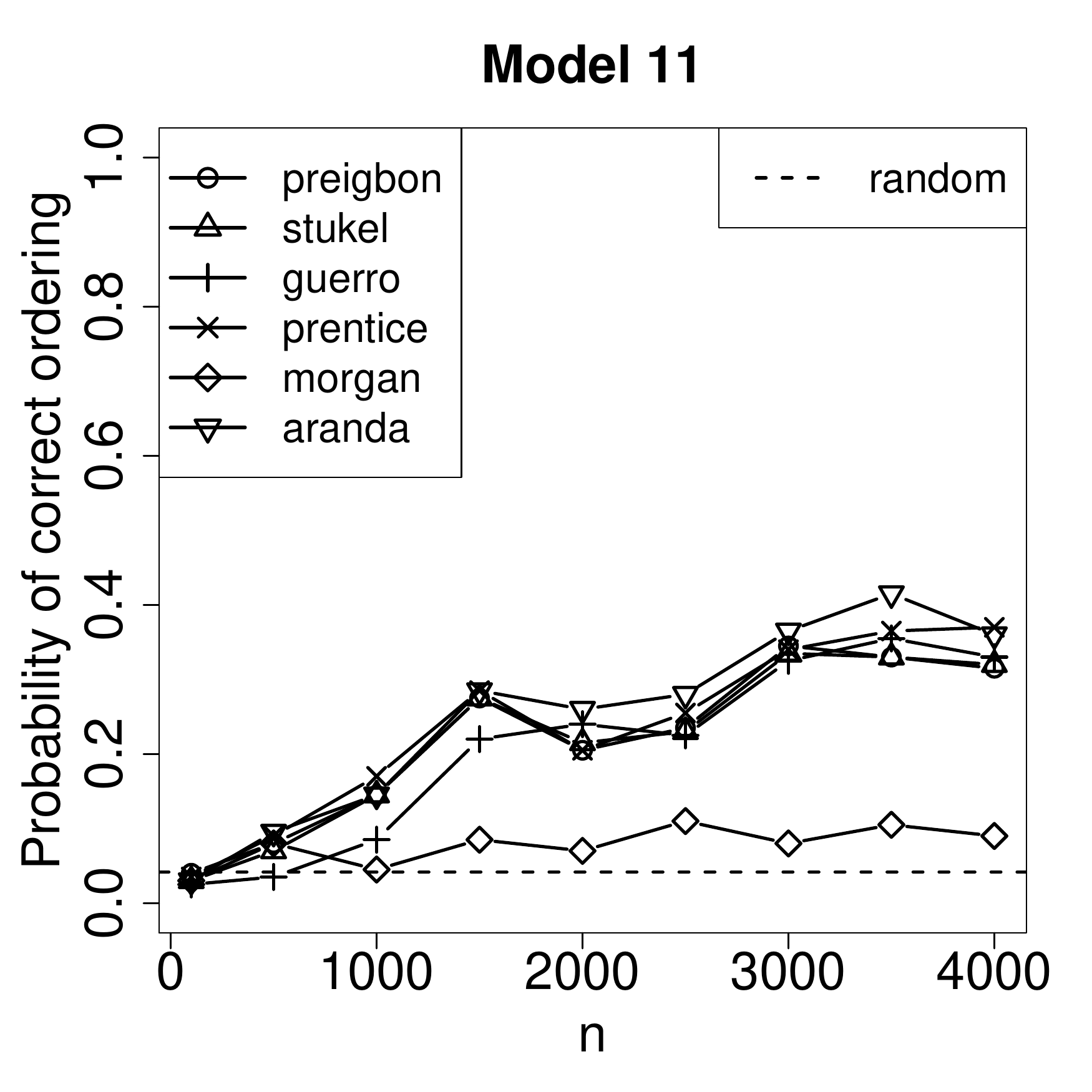} &
\includegraphics[scale=0.28]{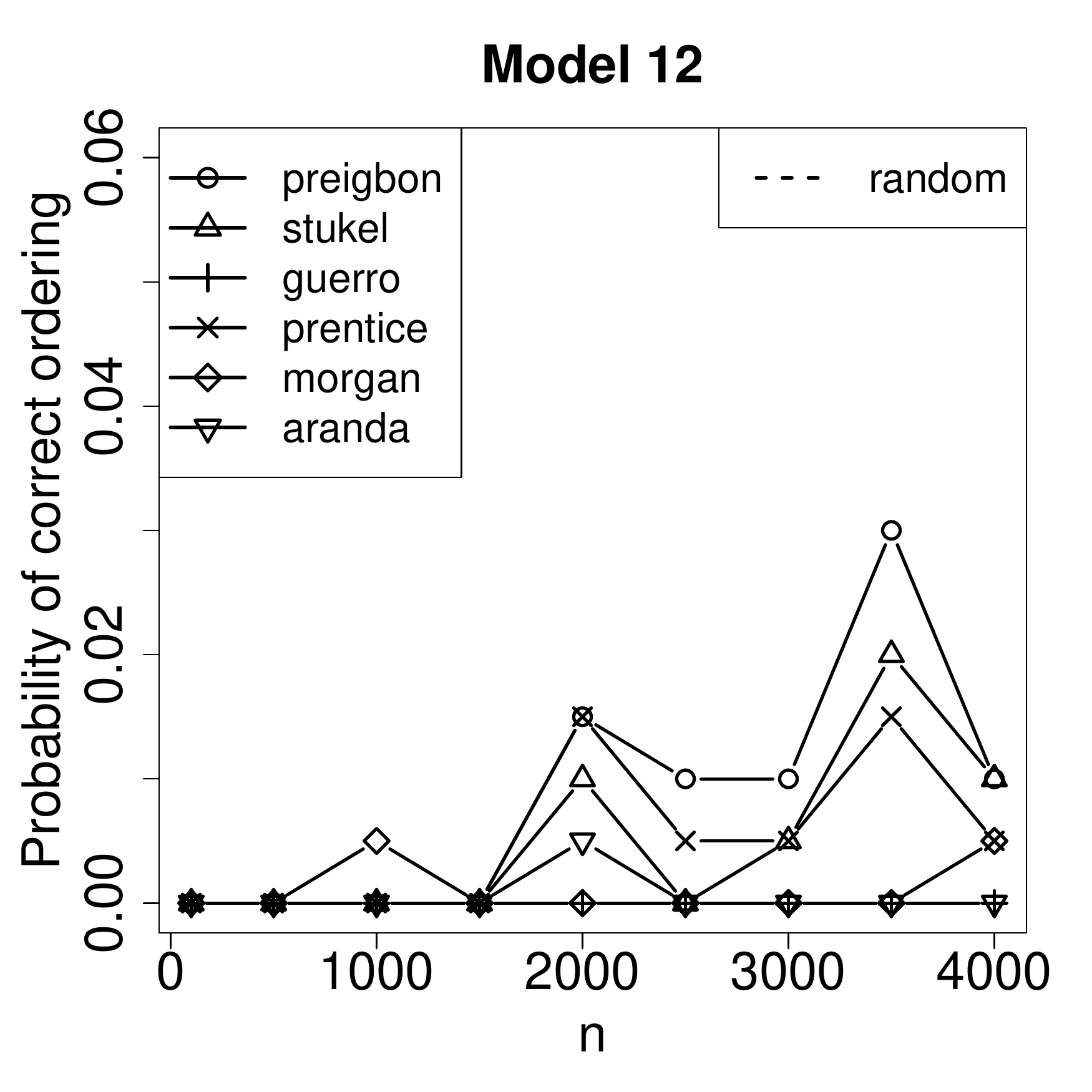} \\
\end{array}$
\end{center}
\caption{Probabilities of finding the correct ordering with respect to the number of training cases. 
The scale for model (M12) is much smaller than for other models.}
\label{fig1}
\end{figure}

\begin{table}
{\scriptsize
\caption{Characteristics of the considered models.}
\label{tab2}
\begin{tabular}{|l|l||l|l|}
\hline
Model & Parameters & Model & Parameters\\
\hline
%%%%%%%%%%%%%%%%%%%%%%%%%%%%%%%%%%%%%%%%
M1
%($p=2$, $K_n=2$)
& 
\begin{tabular}{ll}
  $\theta_1$&$=(0,1)'$ \\
  $\theta_2$&$=(0,1,3)'$
\end{tabular}
&
M7
&
\begin{tabular}{ll}
  $\theta_1$&$=(2, -2, 1)'$ \\
  $\theta_2$&$=(2, -2, 1, 5)' $ \\
  $\theta_3$&$=(2, -2, 1, 5, -5)' $ \\
  $\theta_4$&$=(2, -2, 1, -5, 5, -5)' $ \\
\end{tabular}
\cr
\hline
%%%%%%%%%%%%%%%%%%%%%%%%%%%%%%%%%%%%%%%%
M2
&
\begin{tabular}{ll}
  $\theta_1$&$=(0, 1)'$ \\
  $\theta_2$&$=(0, 1, 5)' $ 
\end{tabular}
&
M8
&
\begin{tabular}{ll}
  $\theta_1$&$=(2, -2, 1)'$ \\
  $\theta_2$&$=(2, -2, 1, 2)' $ \\
  $\theta_3$&$=(2, -2, 1, 2, -2)' $ \\
  $\theta_4$&$=(2, -2, 1, -2, 2, -2)' $ \\
\end{tabular}
\cr
\hline
%%%%%%%%%%%%%%%%%%%%%%%%%%%%%%%%%%%%%%%%
M3
&
\begin{tabular}{ll}
  $\theta_1$&$=(2, -2, 1)'$ \\
  $\theta_2$&$=(2, -2, 1, 5)' $ \\
  $\theta_3$&$=(2, -2, 1, 5, -5)' $ \\
  $\theta_4$&$=(2, -2, 1, -5, 5, -5)' $ \\
  $\theta_5$&$=(2, -2, 1, 5, -5, 5, -5)' $ \\
  $\theta_6$&$=(2, -2, 1, 5, -5, 5, -5, 5)' $ \\
\end{tabular}
&
M9
&
\begin{tabular}{ll}
  $\theta_1$&$=(2, -2, 1)'$ \\
  $\theta_2$&$=(2, -2, 1, 10)' $ \\
  $\theta_3$&$=(2, -2, 1, 10, -10 )' $ \\
  $\theta_4$&$=(2, -2, 1, -10, 10, -10)' $ \\
\end{tabular}
\cr
\hline
%%%%%%%%%%%%%%%%%%%%%%%%%%%%%%%%%%%%%%%%
M4
&
\begin{tabular}{ll}
  $\theta_1$&$=(2, -2, 1)'$ \\
  $\theta_2$&$=(2, -2, 1, 5)' $ \\
  $\theta_3$&$=(2, -2, 1, 5, -5)' $ \\
  $\theta_4$&$=(2, -2, 1, -5, 5, -5)' $ \\
  $\theta_5$&$=(2, -2, 1, 5, -5, 5, -5)' $ \\
\end{tabular}
&
M10
&
\begin{tabular}{ll}
  $\theta_1$&$=(5, -5, 2)'$ \\
  $\theta_2$&$=(5, -5, 2, 5)' $ \\
  $\theta_3$&$=(5, -5, 2, 5, -5)' $ \\
  $\theta_4$&$=(5, -5, 2, -5, 5, -5)' $ \\
\end{tabular}
\cr
\hline
%%%%%%%%%%%%%%%%%%%%%%%%%%%%%%%%%%%%%%%%
M5
&
\begin{tabular}{ll}
  $\theta_1$&$=\mathbf{a}$ \\
  $\theta_2$&$=(\mathbf{a}', 5)' $ \\
  $\theta_3$&$=(\mathbf{a}', 5, -5)' $ \\
  $\theta_4$&$=(\mathbf{a}', -5, 5, -5)' $ \\
  $\mathbf{a}$&$=(1, -1, 1, -1, 1, -1, 1, -1, 1, -1)'$ 
\end{tabular}
&
M11
&
\begin{tabular}{ll}
$\theta_1$&$=\mathbf{a}$ \\ 
$\theta_2$&$=(\mathbf{a}',  -8)'$ \\ 
$\theta_3$&$=(\mathbf{a}',  1, 3)'$ \\
$\theta_4$&$=(\mathbf{a}', 0.5, 5, 10)'$ \\
$\mathbf{a}$&$=(1, -1, 1, -1, 1, -1, 1, -1, 1, -1)'$ 
\end{tabular}
\cr
\hline
%%%%%%%%%%%%%%%%%%%%%%%%%%%%%%%%%%%%%%%%
M6
&
\begin{tabular}{ll}
  $\theta_1$&$=(1, -3, 0.5)'$ \\
  $\theta_2$&$=(1.5, -2.5, 1, 5)' $ \\
  $\theta_3$&$=(2, -2, 1.5, 5, -5)' $ \\
  $\theta_4$&$=(2.5, -1.5, 2, -5, ,5 -5)' $ \\
\end{tabular}
&
M12
&
\begin{tabular}{ll}
$\theta_1$&$=\mathbf{a}$ \\  
$\theta_2$&$=(\mathbf{a}', 5)'$ \\ 
$\theta_3$&$=(\mathbf{a}', 5, -5, )'$ \\ 
$\theta_4$&$=(\mathbf{a}', -5, 5, -5 )'$ \\  
$\theta_5$&$=(\mathbf{a}', 5, -5, 5, -5 )'$ \\  
$\theta_6$&$=(\mathbf{a}', 5, -5, 5, -5, 5  )'$ \\ 
$\theta_7$&$=(\mathbf{a}', 5, -5, 5, -5, 5, -5 )'$ \\  
$\theta_8$&$=(\mathbf{a}', 5, -5, 5, -5, 5, -5, 5  )'$ \\ 
$\theta_9$&$=(\mathbf{a}', 5, -5, 5, -5, 5, -5, 5, -5  )'$ \\ 
$\theta_{10}$&$=(\mathbf{a}', 5, -5, 5, -5, 5, -5, 5, -5, 5 )'$ \\ 
$\mathbf{a}$&$=(1, -1, 1, -1, 1, -1, 1, -1, 1, -1)'$ 
\end{tabular}
\cr
\hline
\end{tabular}

}
\end{table}

\subsection{Consistency of the joint mode selection}
In the second experiment, we illustrate the theoretical result from Theorem \ref{Theorem1} concerning the consistency of the joint mode selection.
In this experiment we use the same models (M1)-(M12) as in the previous section. The models are generated as was described previously and additionally we generate test set containing $200$ observations. For each observation in the test set we check whether the correct mode was selected. The whole procedure is repeated $200$ times, which yields the estimates of the probability of correct mode selection. 
As the correct ordering is unknown in practical applications, we use the procedure described in Algorithm \ref{alg1}, combined with Preigbon method.
Figure \ref{fig2} shows probabilities of correct mode selection with respect to the number of observations in the training set for  the correct ordering, selected ordering and wrong ordering (reversed correct order). To assess the direct effect of ordering on the joint mode estimation, we use "exhaustive inference" in the case of all models, except model (M12), for which "greedy inference" was used, due to computational costs. 
First of all, it is seen that the ordering of labels may affect the probability of correct mode selection, although for some models (e.g. (M3) and (M4)) the differences are not significant.
Secondly, it is seen that Algorithm \ref{alg1}, combined with Preigbon method, performs well in practice- the results are very close to those for correct ordering (for models (M1), (M2), (M5) they practically coincide). 
In the case of model (M12), the differences between the selected ordering and the wrong one are significant, although it is difficult to determine the true ordering exactly (compare Figure \ref{fig1}, model (M12)).

\begin{figure}[ht!]
\begin{center}$
\begin{array}{ccc}
\includegraphics[scale=0.28]{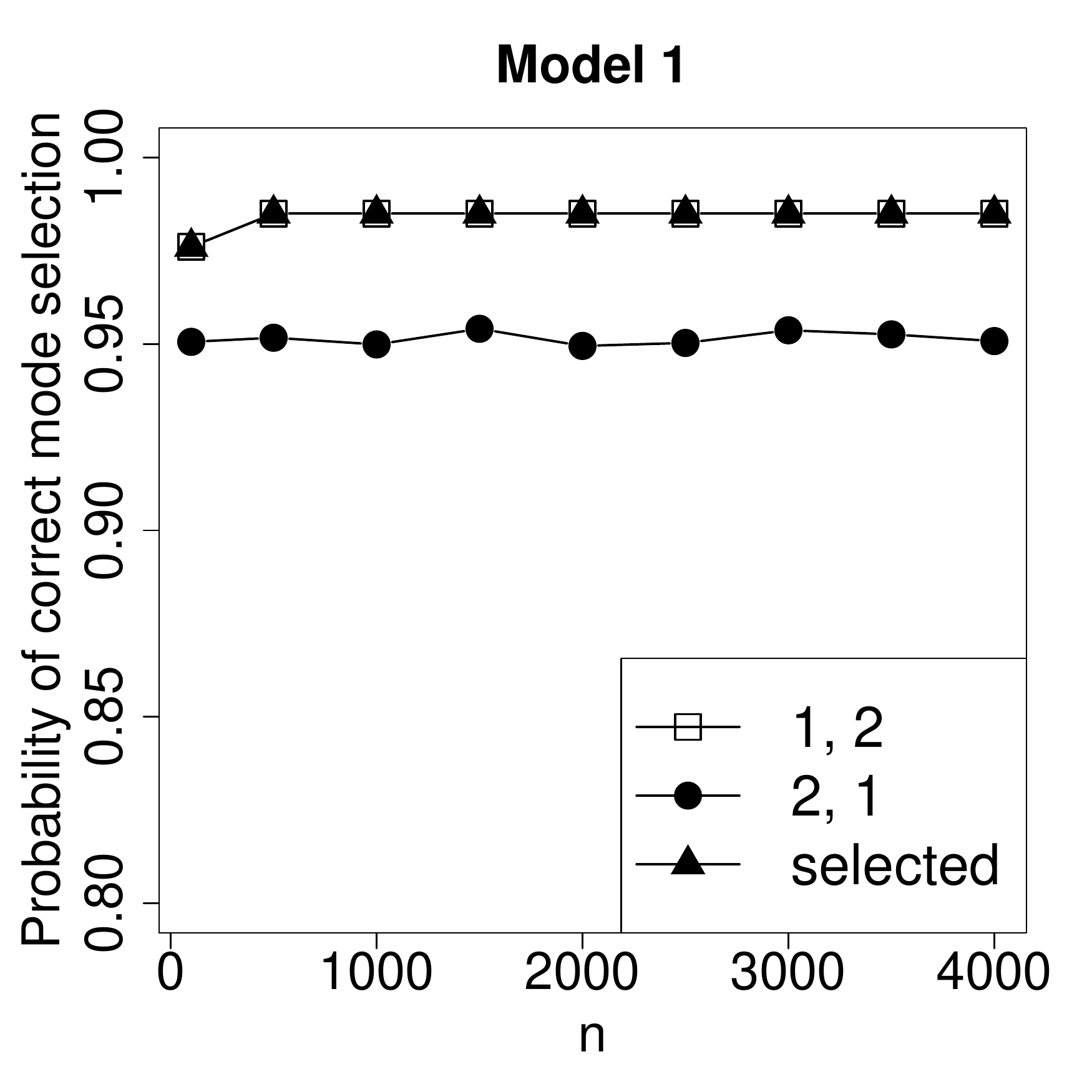} &
\includegraphics[scale=0.28]{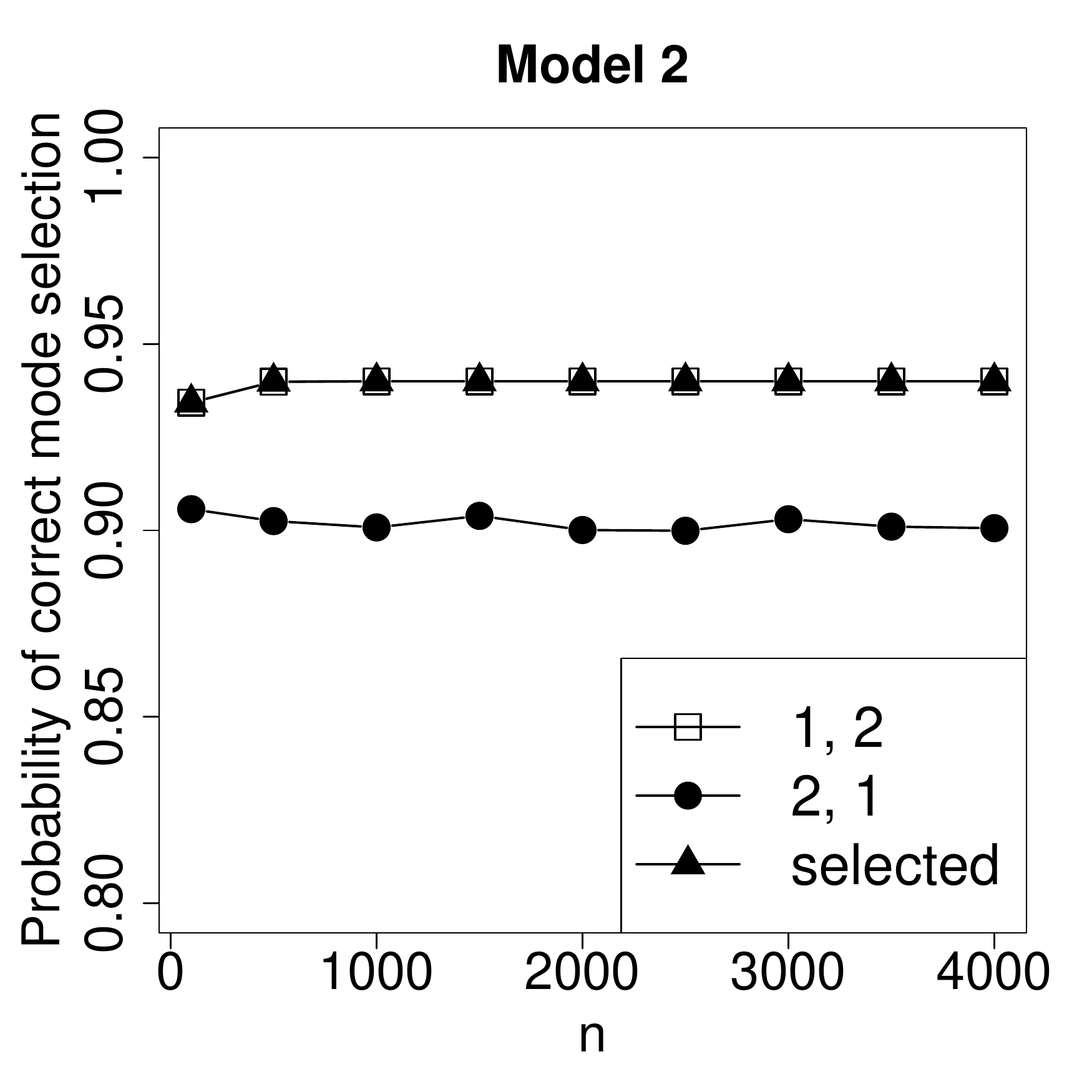} &
\includegraphics[scale=0.28]{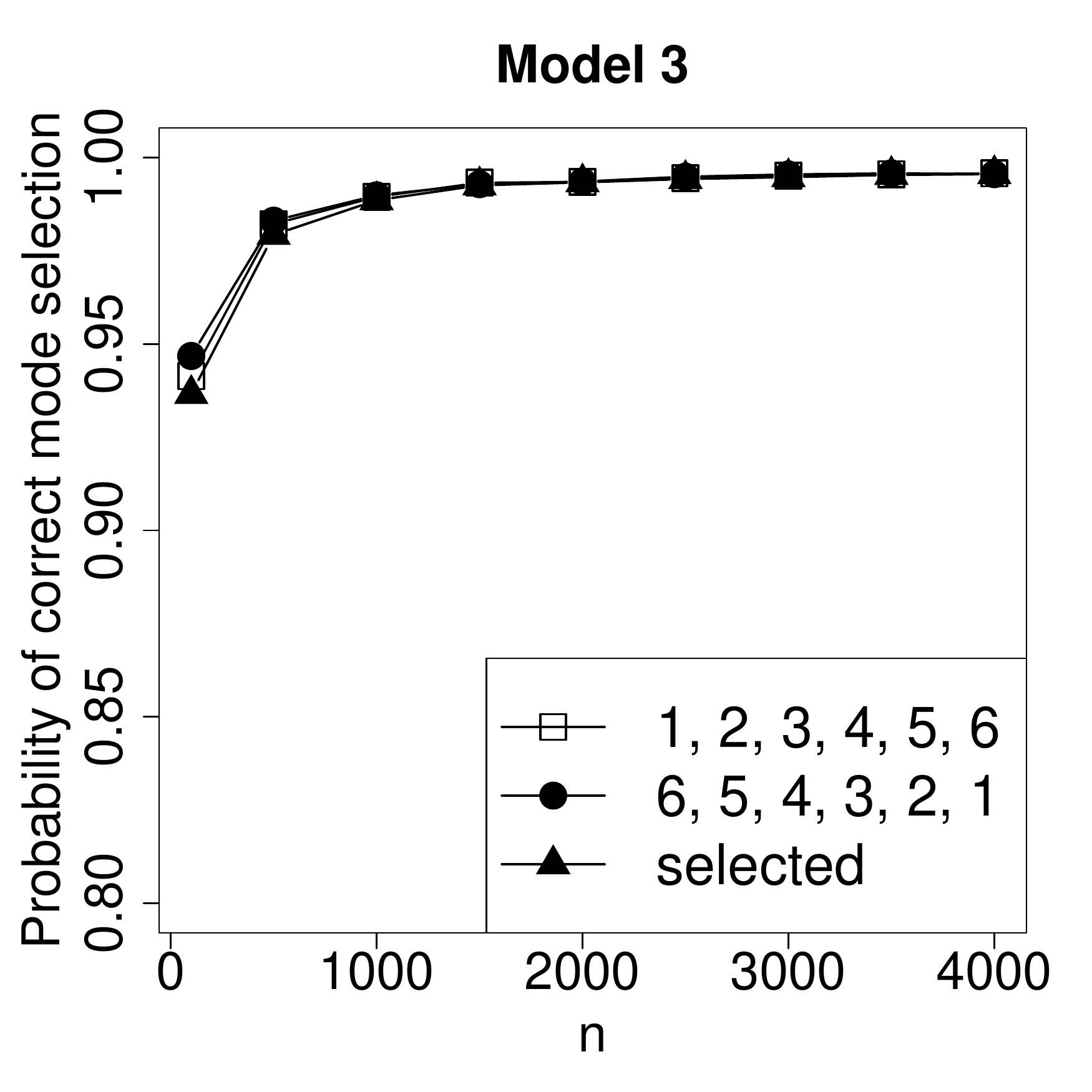} \\
\includegraphics[scale=0.28]{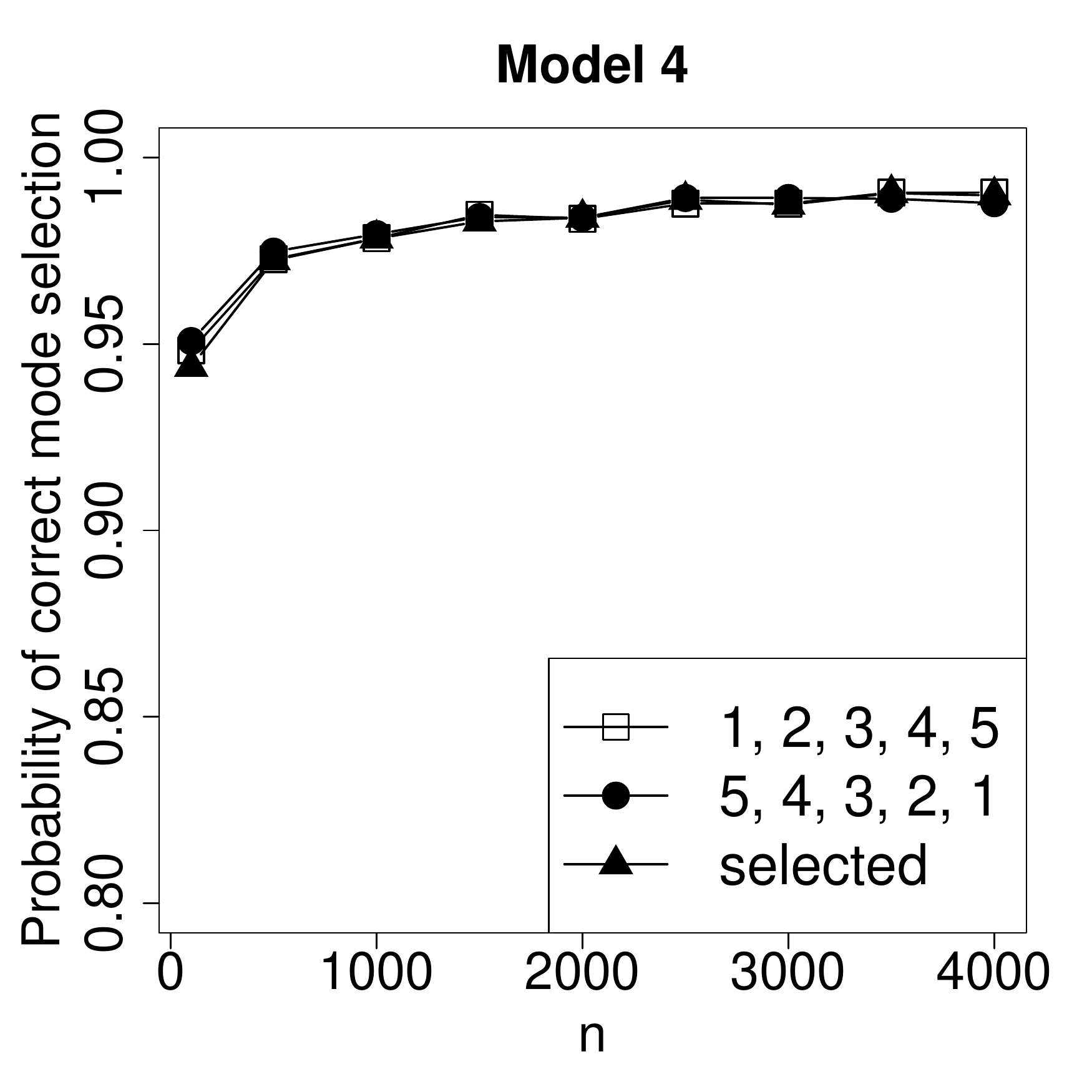} &
\includegraphics[scale=0.28]{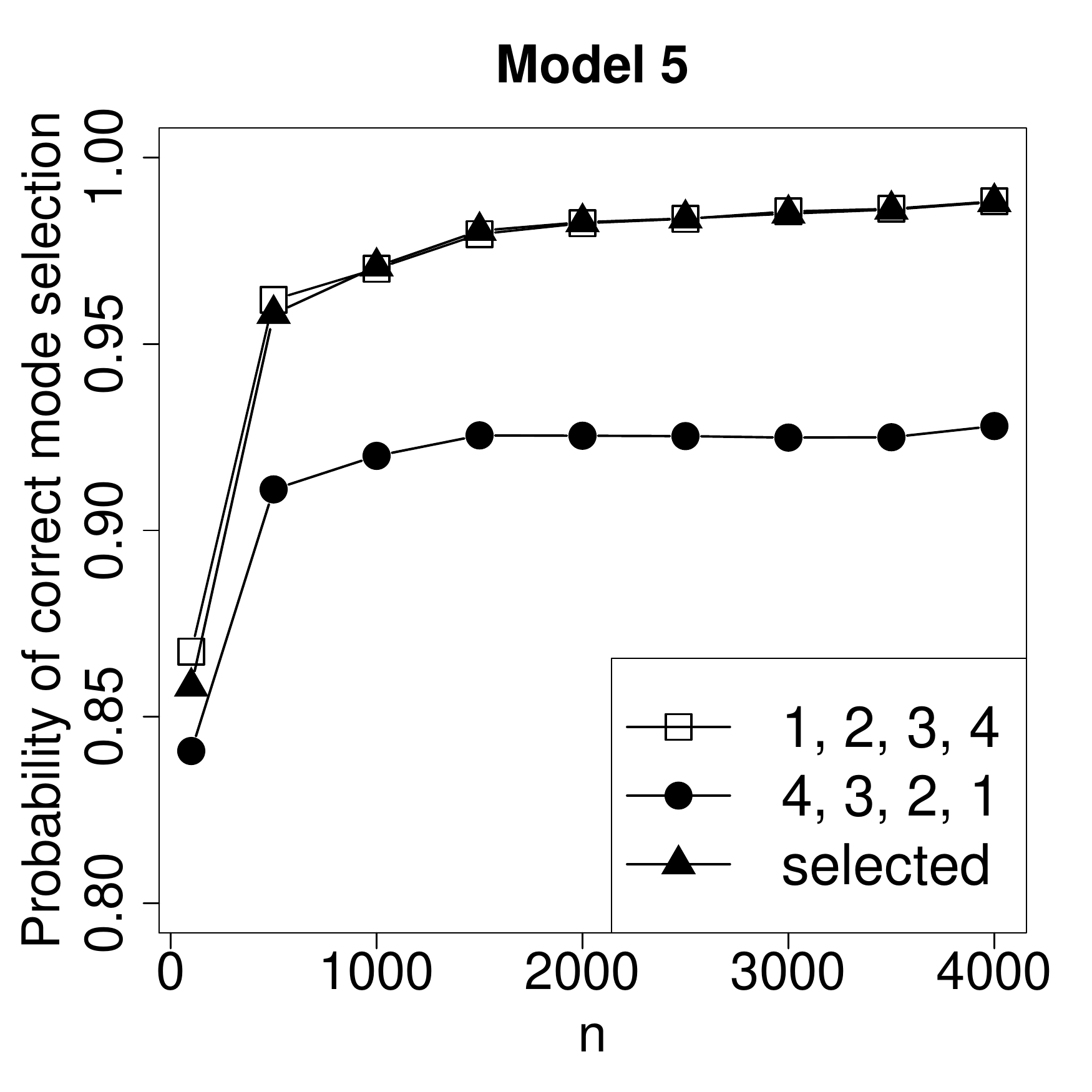} &
\includegraphics[scale=0.28]{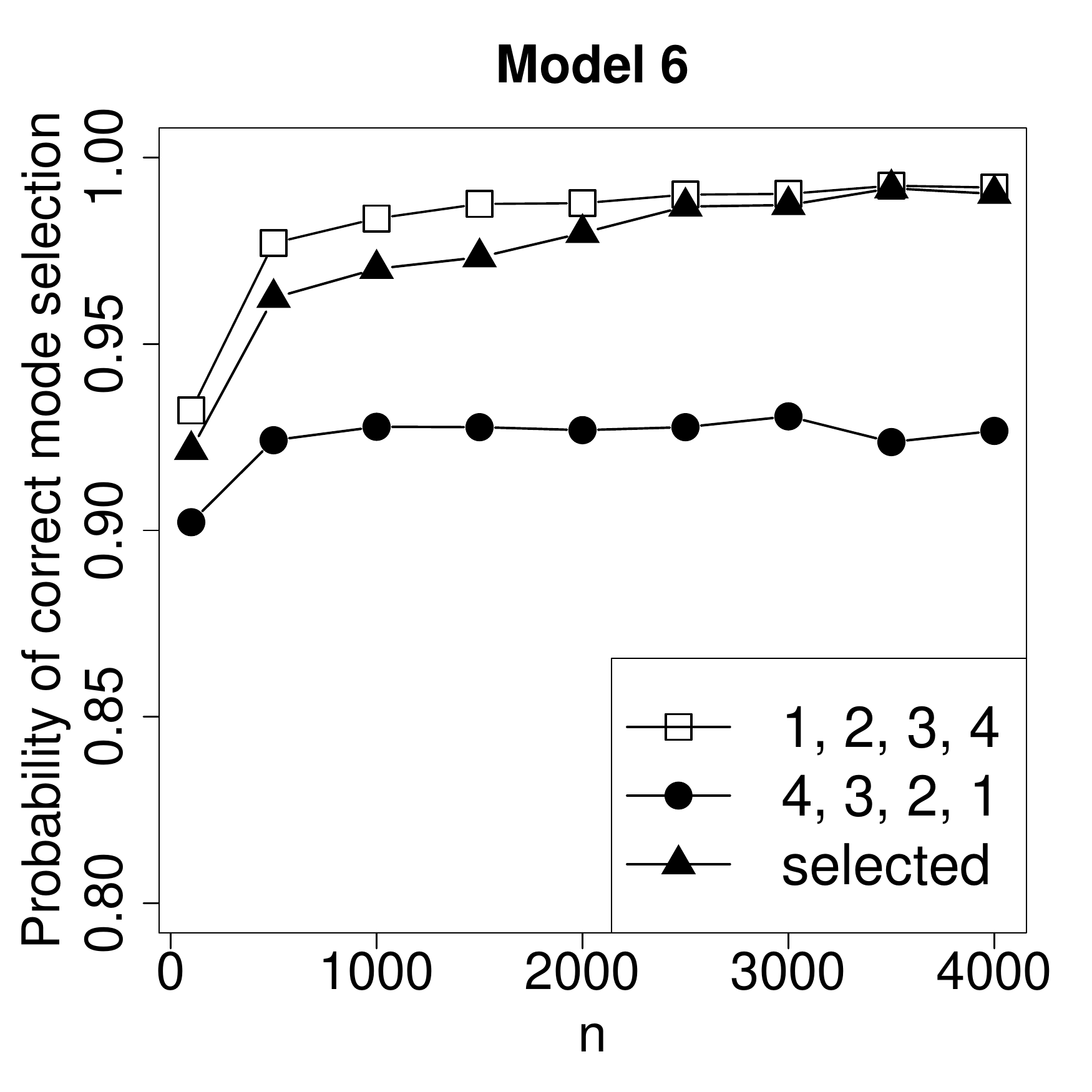} \\
\includegraphics[scale=0.28]{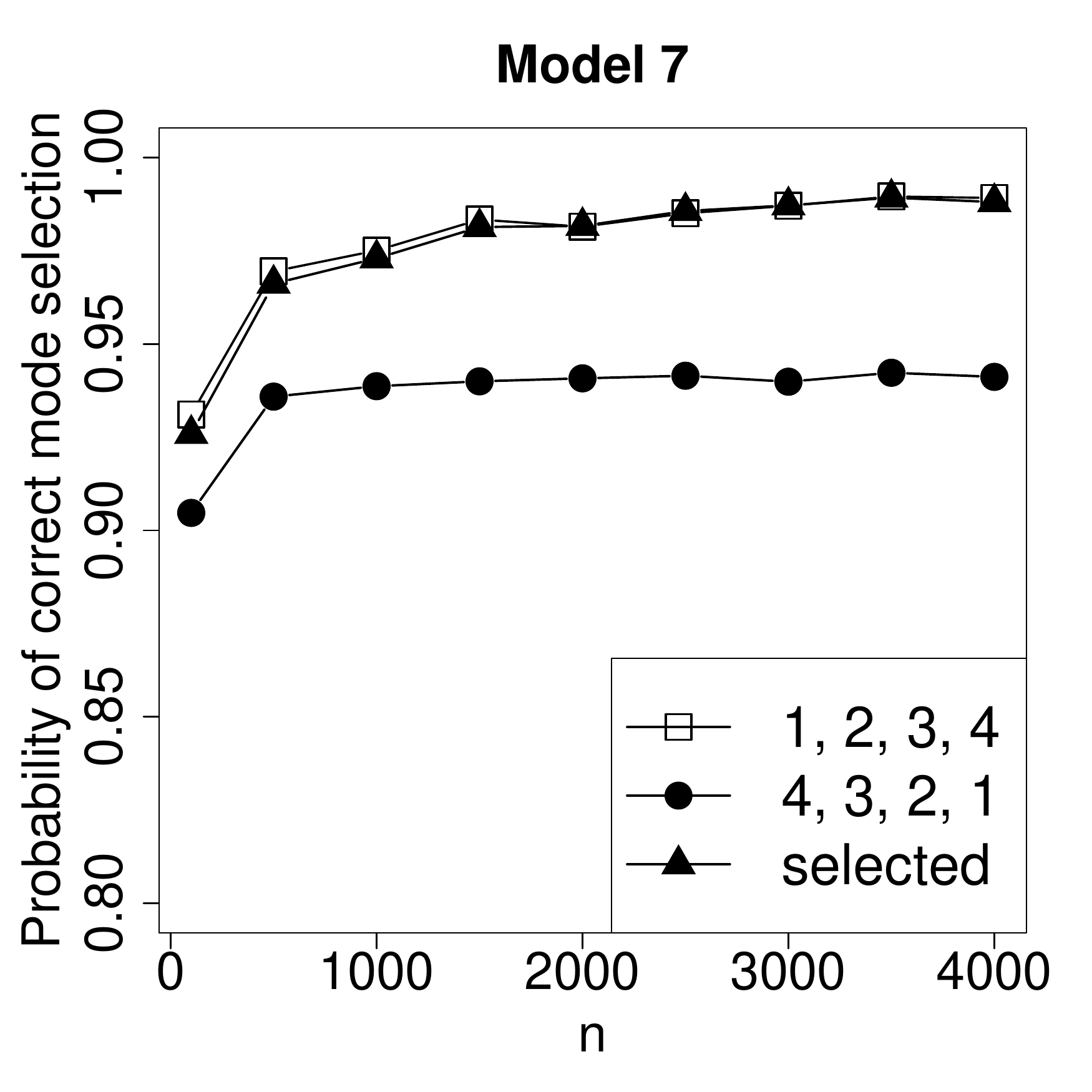} &
\includegraphics[scale=0.28]{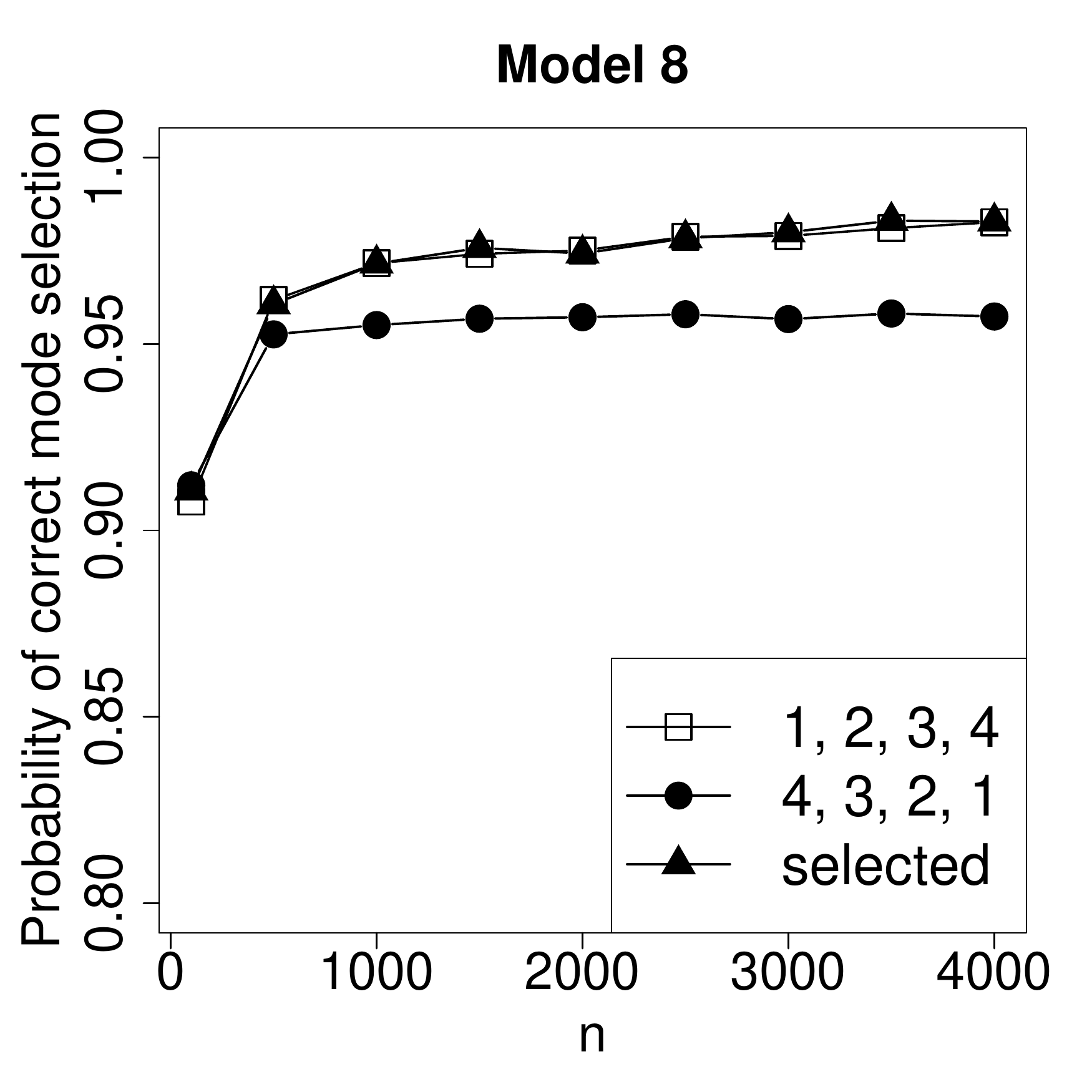} &
\includegraphics[scale=0.28]{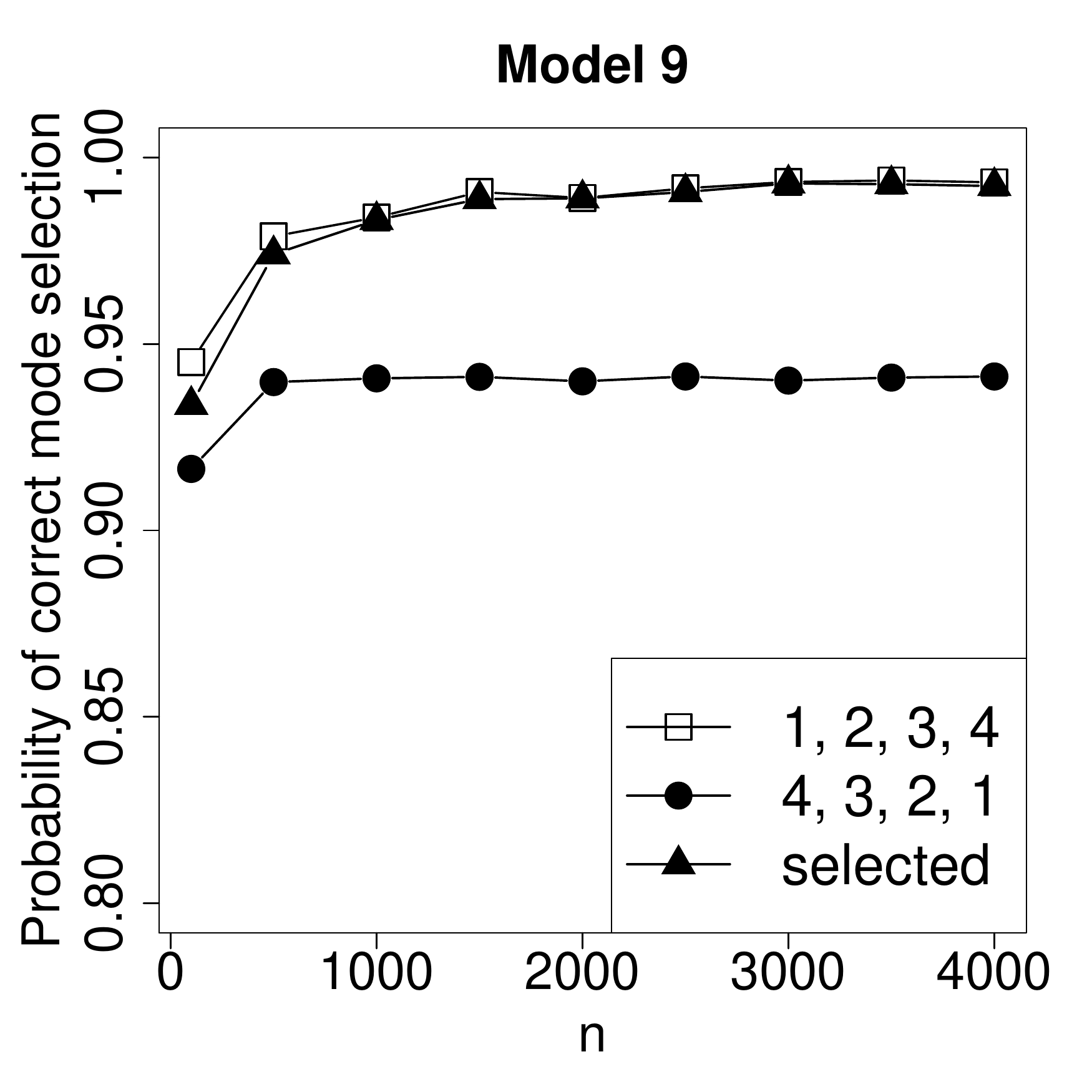} \\
\includegraphics[scale=0.28]{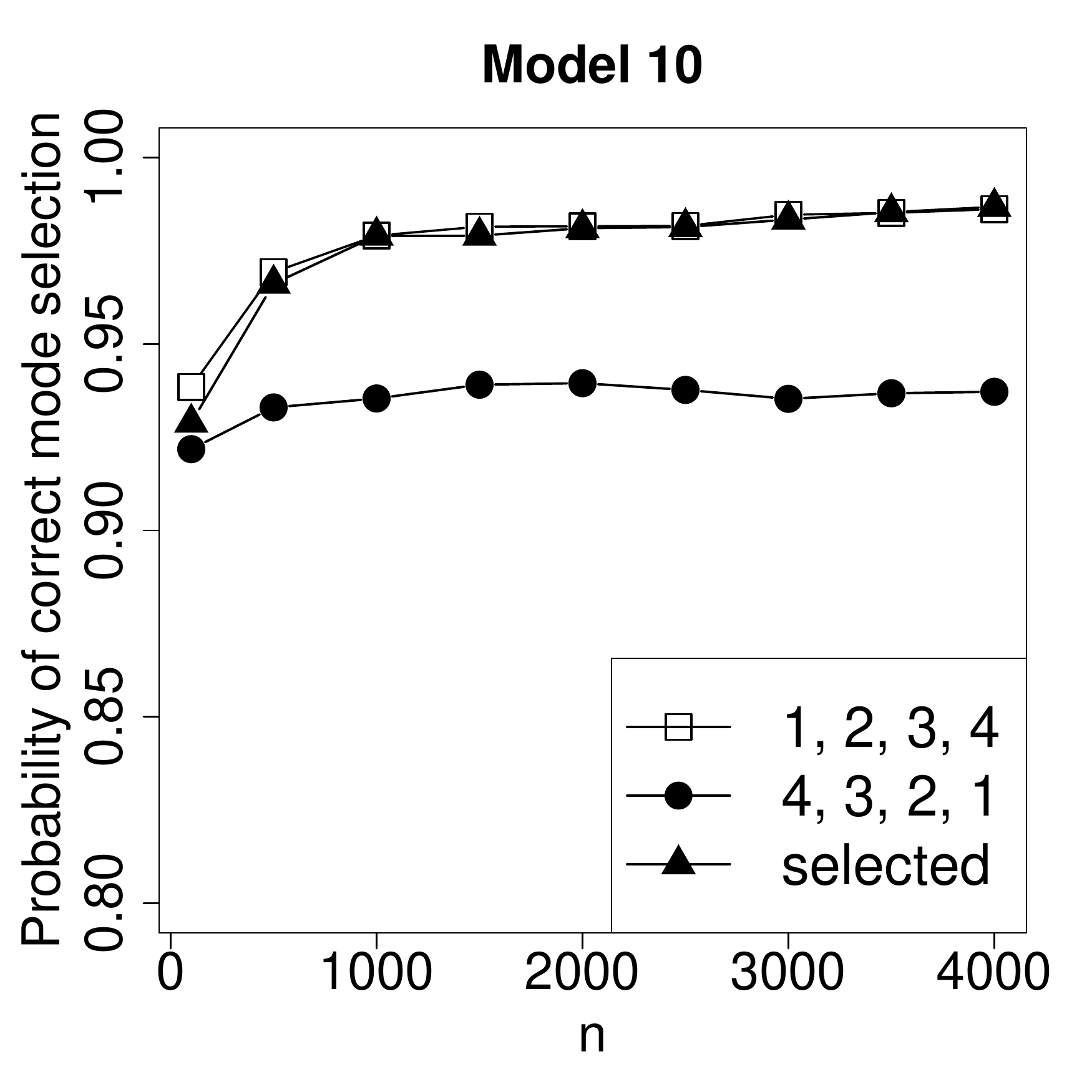} &
\includegraphics[scale=0.28]{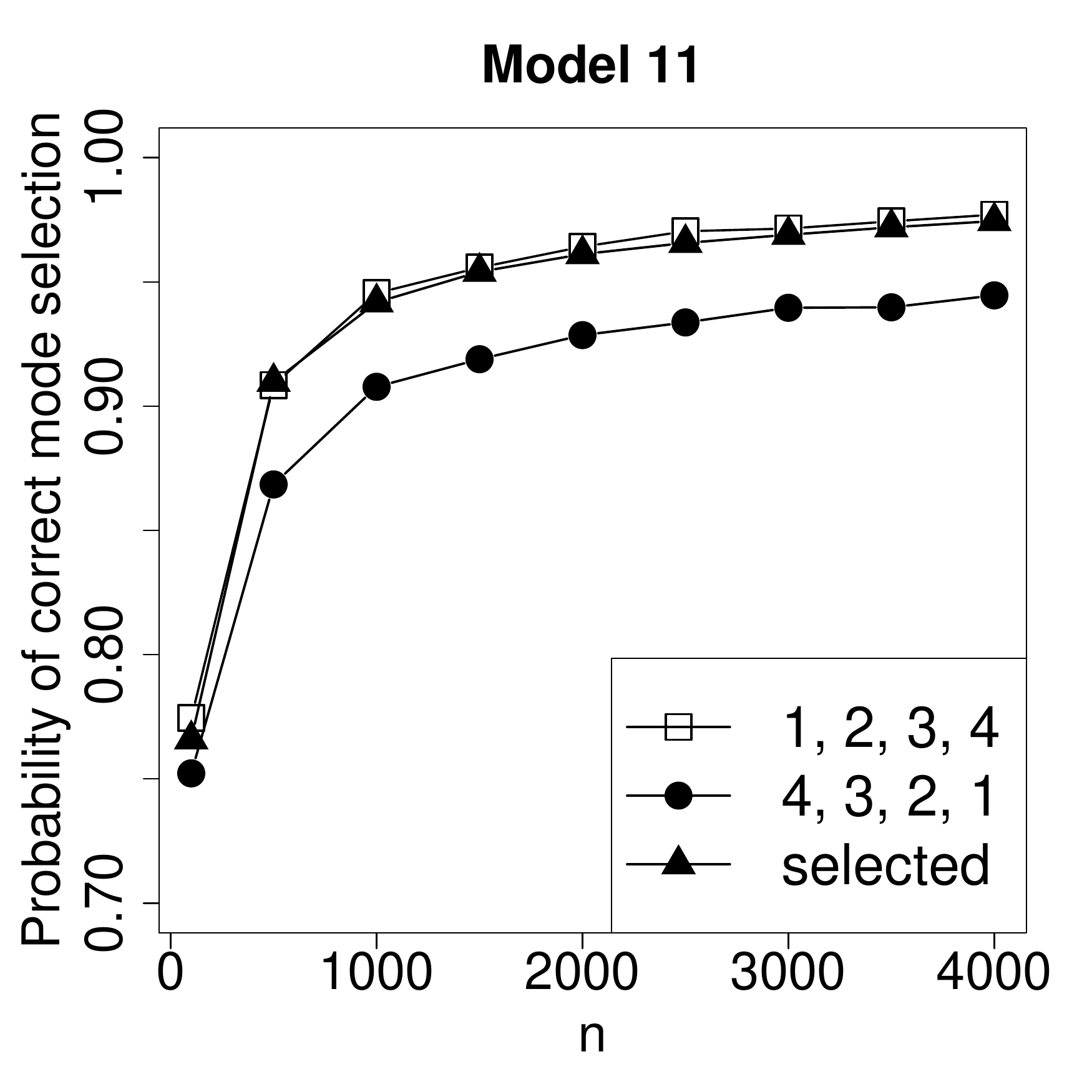} &
\includegraphics[scale=0.28]{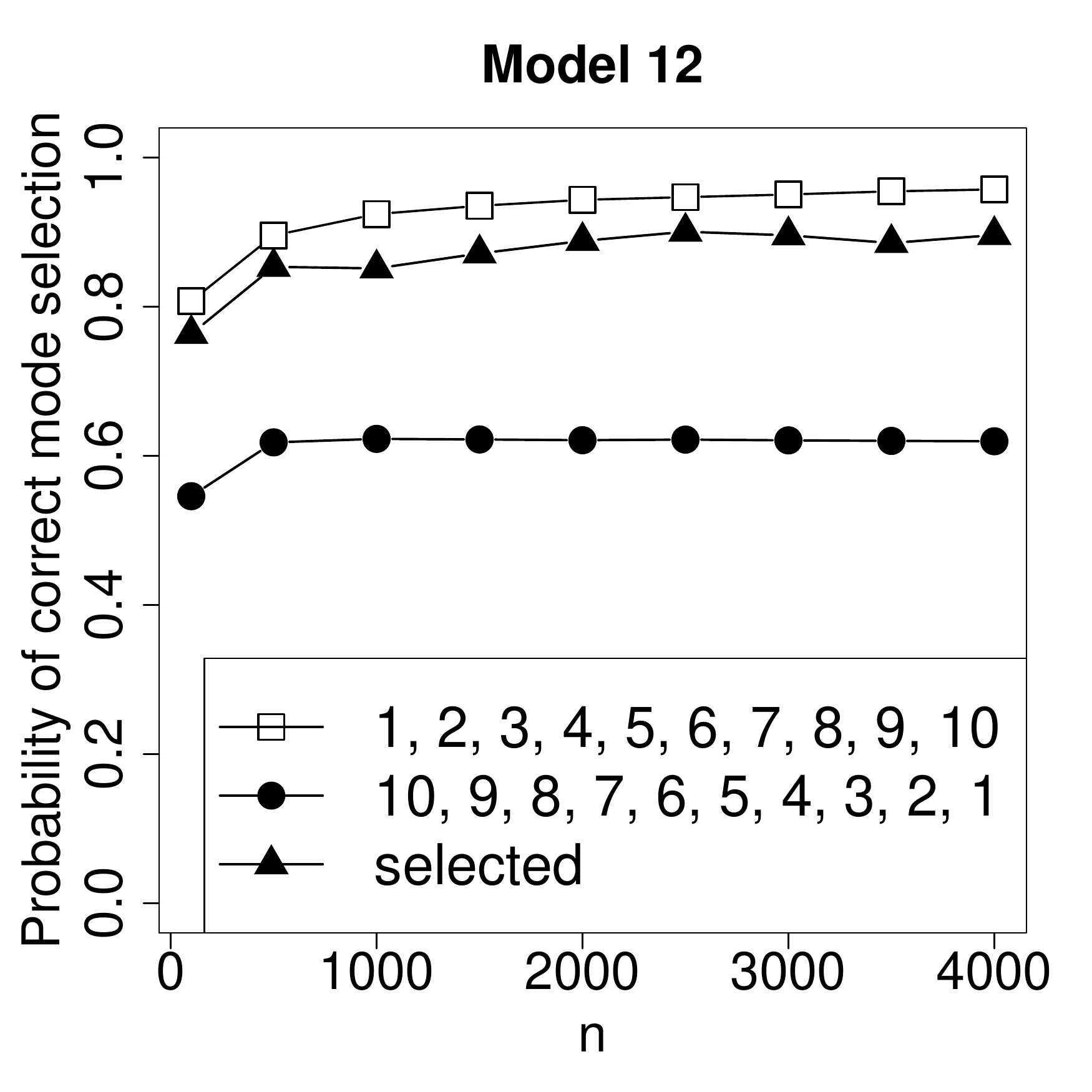}\\
\end{array}$
\end{center}
\caption{Probabilities of correct mode selection with respect to the number of training cases.}
\label{fig2}
\end{figure}

\subsection{Experiments on benchmark data sets}
To investigate the performance of the proposed approach, we also carried out experiments on benchmark datasets.
The datasets are publicly available at \url{http://mulan.sourceforge.net/datasets-mlc.html} (at this website one can also find the detailed description of each data set, as well as the references).
The details of the data sets are summarized in Table \ref{tab3}.
\begin{table}
\caption{Basic statistics for the benchmark datasets.}
\label{tab3}
\begin{tabular}{ccccc}
\hline
Dataset&Domain&$\#$observations&$\#$features&$\#$labels\\
\hline
cal500 	&music& 	502& 	 	68& 	174\\
emotions &	music& 	593& 		72& 	6\\
flags 	&images (toy)& 	194 &	 	19& 	7\\
mediamill 	&video& 	10000& 	 	120& 	101\\
scene& 	images& 	2407& 	 	294& 	6\\
yeast& 	biology& 	2417& 	 	103& 	14 	\\
\hline
\end{tabular}
\end{table}
We compare the following methods.
\begin{itemize}
\item (BR) Binary relevance method, in which we build $K_n$ separate classifiers for each label. Logistic model is used as a single classifier.
\item (CC EX) Classifier chain with the logistic model as a single classifier and "exhaustive inference". The order of fitting models in the chain corresponds to the original ordering of labels in dataset.
\item (CC PREIGBON EX) Classifier chain with the logistic model as a single classifier and "exhaustive inference". The order of fitting models in the chain is determined using Algorithm \ref{alg1} combined with Preigbon method.
\item (CC GR) Classifier chain with the logistic model as a single classifier and "greedy inference". The order of fitting models in the chain corresponds to the original ordering of labels in dataset.
\item (CC PREIGBON GR) Classifier chain with the logistic model as a single classifier and "greedy inference". The order of fitting models in the chain is determined using Algorithm \ref{alg1} combined with Preigbon method.
\end{itemize}
All above methods have been implemented in \texttt{R} system (\cite{Rsystem}), for the purpose of this paper.
The common part of all considered methods is a logistic model.
To reduce the variance of the estimators, we used regularized versions of maximum likelihood estimators, with $l_2$ penalty and small value of  penalty parameter $\lambda=0.001$. 
Let $\y$ be the vector of true labels and $\hat{\y}$ be the predicted vector of labels for given test instance.
We consider the following evaluation measures
\begin{equation*}
\textrm{Hamming}(\y,\hat{\y})=\frac{1}{K_n}\sum_{k=1}^{K_n}\I[y_k=\hat{y}_k],
\end{equation*}
\begin{equation*}
\textrm{Subset accuracy}(\y,\hat{\y})=\I[\y=\hat{\y}],
\end{equation*}
\begin{equation*}
\textrm{Recall}(\y,\hat{\y})=\frac{|\{k:y_k=1,\hat{y}_k=1\}|}{|\{k:y_k=1\}|},
\end{equation*}
\begin{equation*}
\textrm{Precision}(\y,\hat{\y})=\frac{|\{k:y_k=1,\hat{y}_k=1\}|}{|\{k:\hat{y}_k=1\}|},
\end{equation*}
\begin{equation*}
\textrm{F measure}(\y,\hat{\y})=2\cdot\frac{\textrm{Precision}(\y,\hat{\y})\cdot\textrm{Recall}(\y,\hat{\y})}{\textrm{Precision}(\y,\hat{\y})+\textrm{Precision}(\y,\hat{\y})}.
\end{equation*} 
The subset accuracy is maximized by the methods which allow to find the mode of the joint distribution (e.g. classifier chains). On the other hand, the Hamming measure is maximized by the methods which find the marginal modes (BR).
The marginal and joint modes coincide when the labels are conditionally independent given feature vector $\x$, but this is usually not the case.
For further discussion on different evaluation measures and loss functions see in \cite{Dembczynskietal2012}.
Since in many multi-label problems we are interested in predicting the presence of certain properties of the objects (the presence of $k$-th property is usually coded as $y_k=1$), it is worthwhile to consider additional measures: Recall, Precision and F-measure.
For example, in text categorization, $y_k=1$ indicates that the given text has been assigned to $k$-th topic.
 Recall indicates how many labels are correctly predicted as $1$ among those equal $1$, precision shows how many labels are correctly predicted as $1$ among those predicted as $1$ and F-measure is a harmonic mean between recall and precision.  
 
To assess the considered methods, we apply $5$-fold cross-validation. In each cross-validation loop, the models are built on the four, train folds and the prediction is made on the remaining, test fold. The above evaluation measures are averaged over observations in the test fold. Finally, the reported results are averaged over all cross-validation folds. 

We performed two experiments. In the first one, we consider datasets with small or moderate number of labels: emotions, flags and scene. For the remaining datasets (yeast, mediamill, cal500) we have limited the number of labels to 10 in such a way that we keep the most frequent labels (and remove all instances having only relevant or only irrelevant labels). This approach was also used in \cite{Dembczynskietal2010} to test PCC method.
The limited number of labels allows to use "exhaustive inference". Thus we eliminate the effect associated with the type of prediction method and we can directly assess how the ordering of labels affects the results.  The results for the first experiment are shown in Tables \ref{emotions}-\ref{mediamill}. The bold font corresponds to the maximal value in the given column.
It is notable that, application of Algorithm \ref{alg1}, combined with Preigbon method, improves the subset accuracy, relative to original ordering of labels for $5$ out  of $6$ datasets.
This occurs in both cases: for exhaustive and greedy inferences.
 Moreover, in the majority of cases, CC PREIGBON EX outperforms other methods with respect to all measures except Hamming measure. 
The "exhaustive inference" outperforms "greedy inference" which is obvious as in the latter case we explore the limited number of possible labellings.
  We have the highest values of Hamming measure for BR method, which is consistent with observations of other authors and also confirms the theoretical results concerning Hamming measure, given in \cite{Dembczynskietal2012}.

The second experiment was performed for datasets: yeast, mediamill, cal500, with the original number of labels (which varies from $14$ to $174$). For this setting, we only used 3 methods: BR, CC GR and CC GR PREIGBON, as other methods are infeasible for such a large number of labels. 
The results for the second experiment are shown in Tables \ref{yeastall}-\ref{mediamillall}.    
It is seen that also in this case the application of Algorithm \ref{alg1}, combined with Preigbon method, improves the subset accuracy and recall for all three datasets and improves the F-measure for two datasets.
% latex table generated in R 3.0.2 by xtable 1.7-3 package
% Wed Jan 07 19:31:12 2015
% latex table generated in R 3.0.2 by xtable 1.7-3 package
% Wed Jan 07 19:46:44 2015
\begin{table}[ht!]
\footnotesize
\centering
\caption{Results for emotions data set.} 
\label{emotions}
\begin{tabular}{llllll}
  \hline
method & Hamming & Subset accuracy & Recall & Precision & F-measure \\ 
  \hline
BR & {\bf 0.7917} (0.0062) & 0.2478 (0.0092) & 0.6276 (0.015) & 0.6456 (0.0174) & 0.6027 (0.0125) \\ 
  CC EX & 0.7791 (0.0031) & 0.2917 (0.011) & 0.6619 (0.0124) & 0.6416 (0.0106) & 0.6247 (0.0084) \\ 
  CC PREIGBON EX & 0.7889 (0.0036) & {\bf 0.3052} (0.0083) & {\bf 0.6844} (0.0128) & {\bf 0.6627} (0.0103) & {\bf 0.6448} (0.0072) \\ 
  CC GR & 0.776 (0.0051) & 0.2731 (0.0094) & 0.6431 (0.0164) & 0.6405 (0.015) & 0.6126 (0.0139) \\ 
  CC PREIGBON GR & 0.7839 (0.0051) & 0.29 (0.0099) & 0.6661 (0.0148) & 0.6531 (0.0104) & 0.6303 (0.0108) \\ 
   \hline
\end{tabular}
\end{table}
% latex table generated in R 3.0.2 by xtable 1.7-3 package
% Wed Jan 07 19:46:44 2015
\begin{table}[ht!]
\footnotesize
\centering
\caption{Results for scene data set.} 
\label{scene}
\begin{tabular}{llllll}
  \hline
method & Hamming & Subset accuracy & Recall & Precision & F-measure \\ 
  \hline
BR & {\bf 0.8948} (9e-04) & 0.5272 (0.0055) & 0.6597 (0.0032) & 0.6217 (0.0043) & 0.6279 (0.0033) \\ 
  CC EX & 0.8912 (0.0027) & 0.6315 (0.0056) & 0.7035 (0.008) & 0.706 (0.0085) & 0.6967 (0.0078) \\ 
  CC PREIGBON EX & 0.8945 (0.0023) & {\bf 0.6386} (0.0069) & {\bf 0.7145} (0.006) & {\bf 0.7121} (0.0069) & {\bf 0.7051} (0.0063) \\ 
  CC GR & 0.8905 (0.0032) & 0.6278 (0.0079) & 0.701 (0.0104) & 0.7021 (0.0102) & 0.6934 (0.0098) \\ 
  CC PREIGBON GR & 0.8922 (0.0028) & 0.6298 (0.0085) & 0.7062 (0.0078) & 0.7034 (0.0089) & 0.6966 (0.0082) \\ 
   \hline
\end{tabular}
\end{table}
% latex table generated in R 3.0.2 by xtable 1.7-3 package
% Wed Jan 07 19:46:44 2015
\begin{table}[ht!]
\footnotesize
\centering
\caption{Results for flags data set.} 
\label{flags}
\begin{tabular}{llllll}
  \hline
method & Hamming & Subset accuracy & Recall & Precision & F-measure \\ 
  \hline
BR & {\bf 0.7297} (0.0124) & 0.139 (0.0098) & 0.6757 (0.0191) &  {\bf 0.695} (0.0124) & 0.6723 (0.0117) \\ 
  CC EX & 0.7173 (0.0126) & 0.237 (0.0316) & {\bf 0.6785} (0.0154) & 0.6733 (0.0124) & 0.6736 (0.0134) \\ 
  CC PREIGBON EX & 0.7179 (0.0056) & {\bf 0.2625} (0.029) & 0.6749 (0.0138) & 0.6774 (0.0096) & {\bf 0.6741} (0.0111) \\ 
  CC GR & 0.6982 (0.0118) & 0.2012 (0.0156) & 0.6559 (0.0199) & 0.6535 (0.0141) & 0.652 (0.0168) \\ 
  CC PREIGBON GR & 0.7144 (0.0145) & 0.2421 (0.0408) & 0.6733 (0.0166) & 0.6771 (0.0088) & 0.673 (0.0131) \\ 
   \hline
\end{tabular}
\end{table}
% latex table generated in R 3.0.2 by xtable 1.7-3 package
% Wed Jan 07 19:46:44 2015
\begin{table}[ht!]
\footnotesize
\centering
\caption{Results for yeast data set (top $10$ labels).} 
\label{yeast}
\begin{tabular}{llllll}
  \hline
method & Hamming & Subset accuracy & Recall & Precision & F-measure \\ 
  \hline
BR & {\bf 0.7461} (0.004) & 0.1589 (0.0077) & 0.6175 (0.0042) & {\bf 0.6972} (0.0076) & 0.626 (0.0052) \\ 
  CC EX & 0.7352 (0.0044) & {\bf 0.2453} (0.008) & {\bf 0.6572} (0.0074) & 0.6578 (0.0066) & {\bf 0.6345} (0.0069) \\ 
  CC PREIGBON EX & 0.7328 (0.0061) & 0.2379 (0.0098) & 0.6499 (0.0095) & 0.661 (0.0077) & 0.631 (0.0076) \\ 
  CC GR & 0.7255 (0.0047) & 0.2214 (0.0057) & 0.6208 (0.0063) & 0.6428 (0.0083) & 0.6056 (0.007) \\ 
  CC PREIGBON GR & 0.7313 (0.006) & 0.2143 (0.0099) & 0.6343 (0.0094) & 0.6601 (0.0086) & 0.6202 (0.0072) \\ 
   \hline
\end{tabular}
\end{table}
% latex table generated in R 3.0.2 by xtable 1.7-3 package
% Wed Jan 07 19:46:44 2015
\begin{table}[ht!]
\footnotesize
\centering
\caption{Results for cal500 data set (top $10$ labels).} 
\label{cal500}
\begin{tabular}{llllll}
  \hline
method & Hamming & Subset accuracy & Recall & Precision & F-measure \\ 
  \hline
BR & 0.6052 (0.0075) & 0.004 (0.0024) & 0.7193 (0.0075) & 0.6626 (0.0081) & 0.6713 (0.007) \\ 
  CC EX & 0.6016 (0.0087) & 0.0159 (0.0051) & 0.7065 (0.0115) & 0.661 (0.0087) & 0.6649 (0.0091) \\ 
  CC PREIGBON EX & 0.6044 (0.0085) & {\bf 0.016} (0.0051) & 0.7121 (0.0106) & 0.6631 (0.0085) & 0.6679 (0.0089) \\ 
  CC GR & 0.5954 (0.0091) & 0.0139 (0.0074) & 0.6998 (0.009) & 0.6552 (0.0098) & 0.6579 (0.0076) \\ 
  CC PREIGBON GR & 0.608 (0.0083) & 0.01 (0.0055) & {\bf 0.7141} (0.0082) & {\bf 0.6679} (0.0096) & 0.6722 (0.0075) \\ 
   \hline
\end{tabular}
\end{table}
% latex table generated in R 3.0.2 by xtable 1.7-3 package
% Wed Jan 07 19:46:44 2015
\begin{table}[ht!]
\footnotesize
\centering
\caption{Results for mediamill data set (top $10$ labels).} 
\label{mediamill}
\begin{tabular}{llllll}
  \hline
method & Hamming & Subset accuracy & Recall & Precision & F-measure \\ 
  \hline
BR & {\bf 0.8277} (0.0013) & 0.1642 (0.0052) & 0.5844 (0.0035) & {\bf 0.7456} (0.0012) & {\bf 0.6246} (0.0026) \\ 
  CC EX & 0.8165 (0.0023) & 0.2152 (0.0058) & 0.6049 (0.0066) & 0.6836 (0.0049) & 0.6094 (0.0054) \\ 
  CC PREIGBON EX & 0.8186 (0.0021) & {\bf 0.2168} (0.007) & 0.6033 (0.0075) & 0.6873 (0.0035) & 0.6103 (0.0057) \\ 
  CC GR & 0.8199 (9e-04) & 0.2046 (0.004) & 0.6007 (0.005) & 0.703 (0.0017) & 0.6143 (0.0034) \\ 
  CC PREIGBON GR & 0.8211 (0.0025) & 0.2098 (0.0048) & 0.5772 (0.0083) & 0.7087 (0.0052) & 0.6032 (0.0055) \\ 
   \hline
\end{tabular}
\end{table}
%%%%%%%%%%%%%%%%%%%%%%%%%%%%%%%%%%%%%%%%%%%%%%%%%%%%%%%%%%%%%%%%%%%%%%%%%%%%%%%%%%%%%
%%%%%%%%%%%%%%%%%%%%%%%%%%%%%%%%%%%%%%%%%%%%%%%%%%%%%%%%%%%%%%%%%%%%%%%%%%%%%%%%%%%%%
\begin{table}[ht!]
\footnotesize
\centering
\caption{Results for yeast data set (all labels).} 
\label{yeastall}
\begin{tabular}{llllll}
  \hline
method & Hamming & Subset accuracy & Recall & Precision & F-measure \\ 
  \hline
BR & {\bf 0.7959} (0.0032) & 0.1411 (0.0061) & 0.5871 (0.0052) & {\bf 0.6963} (0.0075) & {\bf 0.6083} (0.0051) \\ 
  CC GR & 0.7708 (0.0102) & 0.175 (0.0265) & 0.6033 (0.0091) & 0.6221 (0.0156) & 0.5867 (0.0059) \\ 
  CC PREIGBON GR & 0.7808 (0.0044) & {\bf 0.1812} (0.0052) & {\bf 0.6059} (0.0117) & 0.6526 (0.007) & 0.6017 (0.0077) \\ 
   \hline
\end{tabular}
\end{table}
% latex table generated in R 3.0.2 by xtable 1.7-3 package
% Fri Jan 09 16:26:47 2015
\begin{table}[ht!]
\footnotesize
\centering
\caption{Results for cal500 data set (all labels).} 
\label{call500all}
\begin{tabular}{llllll}
  \hline
method & Hamming & Subset accuracy & Recall & Precision & F-measure \\ 
  \hline
BR &  0.7676 (0.0037) & 0 (0) & 0.3339 (0.008) & 0.2719 (0.0045) & 0.2937 (0.0033) \\ 
  CC GR & {\bf 0.7825} (0.0042) & 0 (0) & 0.3431 (0.0109) & {\bf 0.3017} (0.0087) & {\bf 0.3146} (0.0082) \\ 
  CC PREIGBON GR & 0.7773 (0.0074) & 0 (0) & {\bf 0.3454} (0.0131) & 0.295 (0.0088) & 0.3109 (0.0048) \\ 
   \hline
\end{tabular}
\end{table}
% latex table generated in R 3.0.2 by xtable 1.7-3 package
% Fri Jan 09 16:26:47 2015
\begin{table}[ht!]
\footnotesize
\centering
\caption{Results for mediamill data set (all labels).} 
\label{mediamillall}
\begin{tabular}{llllll}
  \hline
method & Hamming & Subset accuracy & Recall & Precision & F-measure \\ 
  \hline
BR & {\bf 0.9441} (0.0046) & 0 (0) & 0.4538 (0.0045) & {\bf 0.3749} (0.0404) & {\bf 0.3837} (0.0205) \\ 
  CC GR & 0.8989 (0.0026) & 0 (0) & 0.4632 (0.0097) & 0.2175 (0.0063) & 0.2769 (0.0043) \\ 
  CC PREIGBON GR & 0.9118 (0.0068) & {\bf 0.0028} (0.0016) & {\bf 0.4887} (0.0128) & 0.2744 (0.0276) & 0.3191 (0.0182) \\ 
   \hline
\end{tabular}
\end{table}

\section{Conclusions}
\label{Conclusions}
In this paper we studied the large sample properties of the logistic classifier chain (LCC) model.
We found the conditions under which the estimated joint mode is consistent. In particular it follows from Theorem 1 that the number of labels should be relatively small compared with the number of training examples. Secondly, we have shown that when the ordering of building models is incorrect, the consecutive probabilities in the chain are approximated by logistic models as closely as possible with respect to Kullback-Leibler distance. We have shown that the ordering of labels may affect probability of joint mode selection significantly.
The proposed procedure of ordering the labels, based on measures of correct specification, performs well in practice, which is confirmed by experiments on artificial and real datasets.

\section*{Acknowledgements}
Research of Pawe{\l} Teisseyre  was supported by the European Union from resources of the European Social Fund within project 'Information technologies: research and their interdisciplinary applications' POKL.04.01.01-00-051/10-00. 
 
\newpage
\section{Proofs}
\label{Proofs}
\subsection{Proof of Theorem \ref{Theorem1}}
\label{Proof of Theorem1}
We present the proof for $K_n\to\infty$; the proof for the constant $K_n$ goes along the same lines, but requires some simple modifications.
For simplicity, we will write $l_k$, $s_k$ and $H_k$ instead of  $l_{y_k,\z_k}$, $s_{y_k,\z_k}$, $H_{y_k,\z_k}$, respectively.
We will show that $P[\yh\neq \ys]\to 0$, as $n\to\infty$. 
Recall that $\z_k=(\x',\y_{1:(k-1)}')'$.
Using Lemmas \ref{Lemma1} and \ref{Lemma2} we have
\begin{eqnarray}
\label{T1_eq1}
&&
P[\yh\neq \ys]\leq P\left[\max_{\y\in\{0,1\}^{K_n}}|\hat{P}(\y|\x)-P(\y|\x)|>\frac{\epsilon}{2}\right]\leq
\cr
&&
P\left[\max_{\y\in\{0,1\}^{K_n}}\sum_{k=1}^{K_n}\left|\sigma(\z_k'\hat{\theta}_k)^{y_k}[1-\sigma(\z_k'\hat{\theta}_k)]^{1-y_k}-\sigma(\z_k'\theta_k)^{y_k}[1-\sigma(\z_k'\theta_k)]^{1-y_k}\right|>\frac{\epsilon}{2}\right]\leq 
\cr
&&
P\left[\max_{\y\in\{0,1\}^{K_n}}\max_{1\leq k\leq K_n}\left|\sigma(\z_k'\hat{\theta}_k)^{y_k}[1-\sigma(\z_k'\hat{\theta}_k)]^{1-y_k}-\sigma(\z_k'\theta_k)^{y_k}[1-\sigma(\z_k'\theta_k)]^{1-y_k}\right|>\frac{\epsilon}{2K_n}\right]\leq
\cr
&&
\sum_{k=1}^{K_n}P\left[\max_{\y\in\{0,1\}^{K_n}}\left|\sigma(\z_k'\hat{\theta}_k)-\sigma(\z_k'\theta_k)\right|>\frac{\epsilon}{2K_n}\right].
\end{eqnarray}
Using mean value theorem and Schwarz inequality the probability under the above sum can be bounded by
\begin{eqnarray*}
&&
P\left[\max_{\y\in\{0,1\}^{K}}|\z_k'\hat{\theta}_k-\z_k'\theta_k|>\frac{\epsilon}{2K_n}\right]\leq
P\left[||\hat{\theta}_k-\theta_k||>\frac{\epsilon}{2K_{n}^{3/2}}\right],
\cr
\end{eqnarray*}
where the last inequality follows from the fact that $||\z_k||^2\leq 4K_n$, for sufficiently large $n$ (as first $p$ coordinates of $\z_k$ corresponds to a fixed point $\x$).
Denote by $\epsilon_n:=\frac{\epsilon}{2K_{n}^{3/2}}$.
Define $\theta_{k,u}:=\theta_k+u\epsilon_n/2$, for $u\in R^p$ such that $||u||=1$.
Using Taylor expansion and concavity of $l_k$, we have for some point $\theta_k^{*}$ between $\theta_{k,u}$ and $\theta_k$
\begin{eqnarray}
\label{T1_eq3}
&&
P\left[||\hat{\theta}_k-\theta_k||>\epsilon_n\right]\leq 
P\left[\exists u: ||u||=1, l_{k}(\theta_{k,u})-l_{k}(\theta_{k})>0 \right]=
\cr
&&
P\left[\exists u: ||u||=1, u's_{k}(\theta_k)>\epsilon_n u'H_{k}(\theta_k^{*})u/4 \right].
\end{eqnarray}
Using Schwarz inequality we have
\begin{eqnarray}
\label{T1_eq4}
&&
\max_{1\leq k\leq K_n}\max_{1\leq i\leq n}|\z^{(i)'}(\theta_k-\theta^{*}_k)|^2\leq 
\max_{1\leq i\leq n}||\z^{(i)}_{K_n}||^{2}\epsilon_n^2\leq
\max_{1\leq i\leq n}||\x^{(i)}_{K_n}||^{2}\epsilon_n^2+K_n\epsilon_n^2\cp 0,
\end{eqnarray}
where the last convergence follows from Lemma \ref{Lemma6}, Markov inequality and assumption 2. In view of (\ref{T1_eq4}), the assertion of Lemma \ref{Lemma4} is satisfied for all $1\leq k\leq K_n$ and thus the probability in (\ref{T1_eq3}) can be bounded from above by
\begin{eqnarray}
\label{T1_eq5}
&&
P\left[\exists u: ||u||=1, u's_{k}(\theta_k)>\epsilon_n c u'H_{k}(\theta_k)u/4\right]\leq
\cr
&&
P\left[||s_{k}(\theta_k)||>\epsilon_n c \lambda_{\min}[H_{k}(\theta_k)]/4\right]\leq
P\left[||s_{k}(\theta_k)||>\epsilon_n c c_1 n/4 \right],
\end{eqnarray}
where the second inequality follows from assumption 4. 
Let $r_n:=p+K_n-1$.
Using Lemma \ref{Lemma5} and the fact that we have (for some constant $c_2$, with probability tending to $1$)
\begin{equation*}
tr(Z_k'Z_k)\leq r_n c_2 n,
\end{equation*}
the probability in (\ref{T1_eq5}) can be further bounded by
\begin{eqnarray}
\label{T1_eq6}
&&
P\left[||s_k(\theta_k)||^2>\frac{c^2 c_1^2 \epsilon_n^2 n^2 r_n c_2}{16r_n c_2}\right]\leq
P\left[||s_k(\theta_k)||^2>\frac{c^2 c_1^2}{16c_2}\frac{\epsilon_n^2 n}{r_n}tr(Z_k'Z'k)\right]\leq
\cr
&&
\exp\left[-\frac{1}{20}\frac{c^2 c_1^2}{16 c_2}\frac{\epsilon_n^2 n}{r_n}\right].
\end{eqnarray}
It follows from (\ref{T1_eq6}) and assumption 3 that probability in (\ref{T1_eq1}) converges to zero, which ends the proof.
\qed

\subsection{Proof of Corollary \ref{Remark1}}
\label{Proof of Remark1}
The probability in Corollary \ref{Remark1} can be bounded from above as
\begin{eqnarray*}
&&
P(||\hat{{\mathbf\theta}}-{\mathbf\theta}||^2>\epsilon K_n^{-2})=
P\left[\sum_{k=1}^{K_n}||\hat{\theta}_k-\theta_k||^2>\epsilon K_n^{-2}\right]\leq
P(\max_{1\leq k \leq K_n}||\hat{\theta}_k-\theta_k||^2>\epsilon K_n^{-3})\leq
\cr
&&
\sum_{k=1}^{K_n}P(||\hat{\theta}_k-\theta_k||>\epsilon K_n^{-3/2})\to 0,
\end{eqnarray*}
from the proof of Theorem \ref{Theorem1} (see (\ref{T1_eq3})-(\ref{T1_eq6})).
\qed

\subsection{Proof of Theorem \ref{Theorem2}}
\label{Proof of Theorem2}
Consider first term $D_{\gamma=0}$.
It follows from Lemma 4.2.7 and Remark 4.2.3 in \cite{Teisseyre2013}, that assumptions of Theorem \ref{Theorem2} imply assumptions of Theorem 1 in \cite{Fahrmeir1987}, which states that $D_{\gamma=0}\cd \chi^2_{1}$, where $\cd$ denotes convergence in distribution and $\chi^2_{1}$ is a random variable distributed according to chi-squared distribution with $1$ degree of freedom.
Consider term  $D_{\gamma\neq 0}$. It follows from Proposition 4.2.10 in \cite{Teisseyre2013}, that under assumptions of Theorem 2, $P[D_{\gamma\neq 0}>w_n]\to 1$, as $n\to\infty$, for some sequence $w_n\to\infty$. This implies the assertion.
\qed

\subsection{Auxiliary facts}
\begin{Lemma}
\label{Lemma1}
Assume that $P(\y^{*}(\x)|\x)>P(\y|\x)+\epsilon$, for all  
$\y\neq \y^{*}(\x)$ and some $\epsilon>0$. Then $\yh\neq\ys$ implies that 
\begin{equation*}
\max_{\y\in\{0,1\}^{K}}|\hat{P}(\y|\x)-P(\y|\x)|>\frac{\epsilon}{2}.
\end{equation*}
\end{Lemma}
\begin{proof}
When $|P(\ys|\x)-\hat{P}(\ys|\x)|>\frac{\epsilon}{2}$ then the assertion holds.
Consider the case $|P(\ys|\x)-\hat{P}(\ys|\x)|\leq\frac{\epsilon}{2}$. Then we have 
$\hat{P}(\yh|\x)\geq \hat{P}(\ys|\x)$ which implies 
$|P(\yh|\x)-\hat{P}(\yh|\x)|>\frac{\epsilon}{2}$ in view of the assumption. This ends the proof.
\end{proof}
\begin{Lemma}
\label{Lemma2}
Let $|a_k|\leq 1$, $|b_k|\leq 1$, for $k=1,\ldots,K_n$. Then
\begin{equation*}
\left|\prod_{k=1}^{K_n}a_k-\prod_{k=1}^{K_n}b_k\right|\leq \sum_{k=1}^{K_n}|a_k-b_k|.
\end{equation*}
\end{Lemma}
\begin{proof}
Using the following inequality
\begin{eqnarray*}
&&
\left|\prod_{k=1}^{K_n}a_k-\prod_{k=1}^{K_n}b_k\right|=
\left|\prod_{k=1}^{K_n}a_k-b_{K_n}\prod_{k=1}^{K_n-1}a_k+b_{K_n}\prod_{k=1}^{K_n-1}a_k-\prod_{k=1}^{K_n}b_k\right|=
\cr
&&
\left|(a_{K_n}-b_{K_n})\prod_{k=1}^{K_n-1}a_k+b_{K_n}\left(\prod_{k=1}^{K_n-1}a_k-\prod_{k=1}^{K_n-1}b_k\right)\right|\leq
|a_{K_n}-b_{K_n}|+\left|\prod_{k=1}^{K_n-1}a_k-\prod_{k=1}^{K_n-1}b_k\right|.
\end{eqnarray*}
and induction, the assertion follows.
\qed
\end{proof}

Lemmas \ref{Lemma4} and \ref{Lemma6} below are proved in \cite{MielniczukTeisseyre2015}.

\begin{Lemma}
\label{Lemma4}
If $\max_{1\leq i\leq n}|\z^{(i)'}(\gamma-\theta)|\leq c$ then for any vector $u\in R^p$ we have
\begin{equation*}
\exp(-3c)u'H_k(\theta)u\leq u'H_k(\gamma)u \leq \exp(3c)u'H_k(\theta)u.
\end{equation*}
\end{Lemma}

\begin{Lemma}
\label{Lemma6}
The  convergence $\max_{1\leq i\leq n}||\x^{(i)}||\epsilon_n\cp 0$ is equivalent to $||\x^{(n)}||\epsilon_n\cp 0$ for non-decreasing sequence $\epsilon_n\to 0$. 
\end{Lemma}
\bibliography{References}
\bibliographystyle{plainnat}
\end{document}